\documentclass[journal]{IEEEtran}
\usepackage{framed,multirow}
\usepackage{booktabs}       
\usepackage{amssymb}
\usepackage{latexsym}
\usepackage{amsmath}
\usepackage{amsthm}
\usepackage{algorithmic}
\usepackage{algorithm}
\usepackage{graphicx}
\usepackage[margin=0.1pt]{subfig}
\usepackage{hyperref}

\newcommand{\RCOMMENT}[2][.5\linewidth]{%
	\leavevmode\hfill\makebox[#1][l]{//~#2}}

\usepackage{url}
\usepackage{xcolor}

\usepackage{amsmath,amsfonts,bm}
\usepackage{mathtools}

\def\1{\bm{1}}

\DeclareMathAlphabet{\mathsfit}{\encodingdefault}{\sfdefault}{m}{sl}
\SetMathAlphabet{\mathsfit}{bold}{\encodingdefault}{\sfdefault}{bx}{n}

\def\sN{{\mathbb{N}}}

\def\A{{\bf A}}
\def\a{{\bf a}}
\def\B{{\bf B}}
\def\bb{{\bf b}}

\def\g{{\bf g}}

\def\G{{\bf G}}

\def\I{{\bf I}}

\def\N{{\bf N}}

\def\Q{{\bf Q}}

\def\W{{\bf W}}

\def\X{{\bf X}}

\def\x{{ \boldsymbol x}}

\def\y{{\boldsymbol y}}
\def\Z{{\bf Z}}

\def\0{{\bf 0}}
\def\1{{\bf 1}}

\def\DM{{\mathcal D}}

\def\GM{{\mathcal G}}
\def\FM{{\mathcal F}}
\def\IM{{\mathcal I}}

\def\JM{{\mathcal J}}

\def\NM{{\mathcal N}}
\def\OM{{\mathcal O}}

\def\RB{{\mathbb R}}
\def\EB{{\mathbb E}}

\def\Ph{\mbox{\boldmath$\Phi$\unboldmath}}

\newcommand{\ueq}[1][]{%
	\if\relax\detokenize{#1}\relax
	\sbox0{$\underbrace{=}_{}$}%
	\mathrel{\mathmakebox[\wd0]{=}}
	\else
	\mathrel{\underbrace{=}_{\mathclap{#1}}}
	\fi}

\newtheorem{assumption}{Assumption}

\newtheorem{lem}{Lemma}
\newtheorem{defi}{Definition}
\newtheorem{theorem}{Theorem}
\newtheorem{remark}{Remak}
\newtheorem{corollary}{Corollary}
\allowdisplaybreaks[4]

\ifCLASSINFOpdf
\else
\fi
\begin{document}
%
\title{Communication-Efficient \\ Local Decentralized SGD Methods}
%
%
%
\author{Xiang Li,
        Wenhao Yang,
        Shusen Wang,
        Zhihua Zhang
\thanks{Xiang Li is with 
	School of Mathematical Sciences, Peking University, Beijing, 100871, China. E-mail: lx10077@pku.edu.cn.}
\thanks{Wenhao Yang is with Academy for Advanced Interdisciplinary Studies, Peking University, Beijing, 100871, China. E-mail: yangwhsms@gmail.com.}
\thanks{Shusen Wang is with Department of Computer Science, Stevens Institute of Technology, Hoboken, NJ 07030, USA. E-mail: shusen.wang@stevens.edu.}
\thanks{Zhihua Zhang is with 
	School of Mathematical Sciences, Peking University, Beijing, 100871, China. E-mail: zhzhang@math.pku.edu.cn.}
\thanks{Manuscript resubmitted March 5, 2021.}}

%
%

\markboth{Journal of \LaTeX\ Class Files,~Vol.~14, No.~8, August~2015}%
{Shell \MakeLowercase{\textit{et al.}}: Bare Demo of IEEEtran.cls for IEEE Journals}
%



\maketitle

\begin{abstract}
Recently, the technique of local updates is a powerful tool in centralized settings to improve communication efficiency via periodical communication.
For decentralized settings, it is still unclear how to efficiently combine local updates and decentralized communication.
In this work, we propose an algorithm named as LD-SGD, which incorporates arbitrary update schemes that alternate between multiple Local updates and multiple Decentralized SGDs, and provide an analytical framework for LD-SGD.  
Under the framework, we present a sufficient condition to guarantee the convergence. 
We show that LD-SGD converges to a critical point for a wide range of update schemes when the objective is non-convex and the training data are non-identically independent distributed.
Moreover, our framework brings many insights into the design of update schemes for decentralized optimization.
As examples, we specify two update schemes and show how they help improve communication efficiency.
Specifically, the first scheme alternates the number of local and global update steps. 
From our analysis, the ratio of the number of local updates to that of decentralized SGD trades off communication and computation.
The second scheme is to periodically shrink the length of local updates.
We show that the decaying strategy helps improve communication efficiency both theoretically and empirically.
\end{abstract}

\begin{IEEEkeywords}
Distributed Optimization, Federated Learning, Local Updates, Communication Efficiency
\end{IEEEkeywords}

%
\IEEEpeerreviewmaketitle

\section{Introduction}
%
%
%
%
\IEEEPARstart{W}{e} study distributed optimization where the data are partitioned among $n$ worker nodes; the data are not necessarily identically distributed.
We seek to learn the model parameter (aka optimization variable) $\x \in \RB^d$ by solving the following distributed empirical risk minimization problem:
\begin{small}
	\begin{equation}
	\label{eq:goal}
	\min_{\x \in \RB^d} \ f(\x) := 
	\frac{1}{n} \sum_{k=1}^n f_k(\x),
	\end{equation}
\end{small}%
where $f_k(\x) := \EB_{\xi \sim \DM_k} \big[ F_{k} \left( \x ; \xi\right) \big]$ and $\DM_k$ is the distribution of data on the $k$-th node with $k \in [n] := \{1, \cdots, n\}$. 
Here $\xi$ denotes by a sample point for simplicity and can be extended to a batch of data~\cite{li2014efficient}.
Such a problem is traditionally solved under centralized optimization paradigms such as parameter servers \cite{li2014scaling}.
Federated Learning (FL), which often has a central parameter server, enables massive edge computing devices to jointly learn a centralized model while keeping all local data localized~\cite{konevcny2015federated,mcmahan2017communication,konevcny2017stochastic,li2019federated,sahu2019federated}. 
As opposed to  centralized optimization, decentralized optimization lets every worker node collaborate only with their neighbors by exchanging information.
A typical decentralized algorithm works in this way: a node collects its neighbors' model parameters $\x$, takes a weighted average, and then performs a (stochastic) gradient descent to update its local parameters~\cite{lian2017can}.
Decentralized optimization can outperform the centralized under specific settings \cite{lian2017can}.


Decentralized optimization, as well as the centralized, suffers from high communication costs.
The communication cost is the bottleneck of distributed optimization when the number of model parameters or the number of worker nodes are large.
It is well known that deep neural networks have a large number of parameters. 
For example, ResNet-50 \cite{he2016deep} has 25 million parameters, so sending $\x$ through a computer network can be expensive and time-consuming. 
Due to modern big data and big models, a large number of worker nodes can be involved in distributed optimization, which further increases the communication cost.
The situation can be exacerbated if worker nodes in distributed learning are remotely connected, which is the case in edge computing and other types of distributed learning.

In recent years, to directly save communication, many researchers let more local updates happen before each synchronization in centralized settings.
A typical and famous example is Local SGD~\cite{mcmahan2017communication,lin2018don,stich2018local,wang2018cooperative,wang2018adaptive}.
As its decentralized counterpart, Periodic Decentralized SGD (PD-SGD) alternates between a fixed number of local updates and one step of decentralized SGD~\cite{wang2018cooperative}.
However, its update scheme is too rigid to balance the trade-off between communication and computation efficiently~\cite{wang2019matcha}.
It is still unclear how to combine local updates and decentralized communications efficiently in decentralized 
settings.

To answer the question, in the paper, we propose a meta algorithm termed as LD-SGD, which is able to incorporate arbitrary update schemes for decentralized optimization.
We provide an analytical framework, which sheds light on the relationship between convergence and update schemes. 
We show that LD-SGD converges with a wide choice of communication patterns for non-convex stochastic optimization problems and non-identically independently distributed training data (i.e., $\DM_1 , \cdots , \DM_n$ are not the same).

We then specify two update schemes to illustrate the effectiveness of LD-SGD.
For the first scheme, we let LD-SGD alternate (i.e., $I_1$ steps of) multiple local updates and multiple (i.e., $I_2$ steps of) decentralized SGDs; see the illustration in Figure~\ref{fig:decentralized} (b).
A reasonable choice of $I_2$ could better trade off the balance between communication and computation both theoretically and empirically.

We observe that more local computation (i.e., large $I_1/I_2$) often leads to higher final errors and less test accuracy.
Therefore, in the second scheme, we propose and analyze a decaying strategy that periodically halves $I_1$.
From our framework, we theoretically verify the efficiency of the strategy.
Finally, as an extension, we test LD-SGD  

\section{Related Work}


\subsection{Decentralized stochastic gradient descent (D-SGD)}
Decentralized (stochastic) algorithms were used as compromises when a powerful central server is not available.
They were studied as consensus optimization in the control community~\cite{ram2010asynchronous,yuan2016convergence,sirb2016consensus}.
\cite{lian2017can} justified the potential advantage of D-SGD over its centralized counterpart. 
D-SGD not only reduces the communication cost but achieves the same linear speed-up as centralized counterparts when more nodes are available~\cite{lian2017can}.
This promising result pushes the research of distributed optimization from a sheer centralized mechanism to a more decentralized pattern~\cite{lan2017communication,tang2018d,koloskova2019decentralized,wang2019matcha,luo2019heterogeneity}.

\subsection{Communication efficient algorithms}
The current methodology towards communication-efficiency in distributed optimization could be roughly divided into three categories. 
The most direct approach is to reduce the size of the messages through gradient compression or sparsification \cite{seide2014bit,lin2017deep,zhang2017zipml,tang2018communication,wang2018atomo,horvath2019stochastic}.
An orthogonal one is to pay more local computation for less communication, e.g.,
one-shot aggregation \cite{zhang2013communication,zhang2015divide,lee2017communication,lin2017distributed,wang2019sharper}, primal-dual algorithms \cite{smith2016cocoa,smith2017federated,hong2018gradient} and distributed Newton methods
\cite{shamir2014communication,zhang2015disco,reddi2016aide,wang2018giant,mahajan2018efficient}.
Beyond them, a simple but powerful method is to reduce the communication frequency by allowing more local updates~\cite{zinkevich2010parallelized,stich2018local,lin2018don,you2018imagenet,wang2018cooperative}, which we focus on in this paper. 
The last category is the push-sum methods that is popular  in control theory community~\cite{assran2019stochastic,wang2019slowmo}.

\subsection{Federated optimization}
The optimization problem implicit in FL is referred to as Federated Optimization (FO).
One of the biggest difference that differs FO from previous distributed optimization is that the training data is generated independently but according to different distributions.
One typical optimization method for FO is is Federated Averaging (FedAvg) \cite{konevcny2016federated,mcmahan2017communication, li2019convergence}, which is a centralized optimization method and is also referred to as Local SGD.
In every iteration of FedAvg, a small subset of nodes is activated, and it alternates between multiple SGDs and sends updated parameters to the central server.
PD-SGD is an extension of FedAvg (or Local SGD) towards decentralized optimization~\cite{wang2018cooperative,haddadpour2019convergence}.
MATCHA~\cite{wang2019matcha} extends PD-SGD to a more federated setting by only activating a random subgraph of the network topology each round.
Our work can be viewed as an attempt to generalize FedAvg to decentralized non-iid settings.

\subsection{Most related work}
Our work is most closely related with ones in~\cite{wang2018cooperative,wang2019matcha}. Specifically, 
\cite{wang2018cooperative} proposed PD-SGD that can also combine decentralization and local updates.
However, they only considered the case of one step of decentralized SGD after a fixed number of local updates.
Moreover, they analyzed PD-SGD by assuming all worker nodes have access to the underlying distribution (hence data are identically distributed).

MATCHA~\cite{wang2019matcha} makes communication happen only among a random small portion of worker nodes at each round.\footnote{Specifically, they first decompose the network topology into joint matchings (or subgraphs), then randomly activates a small portion of matchings.}
When no node is activated, local updates come in.
Consequently, the theory of MATCHA is formulated for random connection matrices (i.e., $\W$ in our case) and does not straightforwardly extend to a deterministic sequence of $\W$.
Our work mainly studies a deterministic sequence of $\W$ but could also extend to random sequences.

\section{Notation and Preliminaries}
\label{sec:notation}

\subsection{Decentralized system}


\begin{figure*}[ht]
	\centering
	\subfloat[5 nodes in a decentralized system]{
		\includegraphics[width=0.49\textwidth] {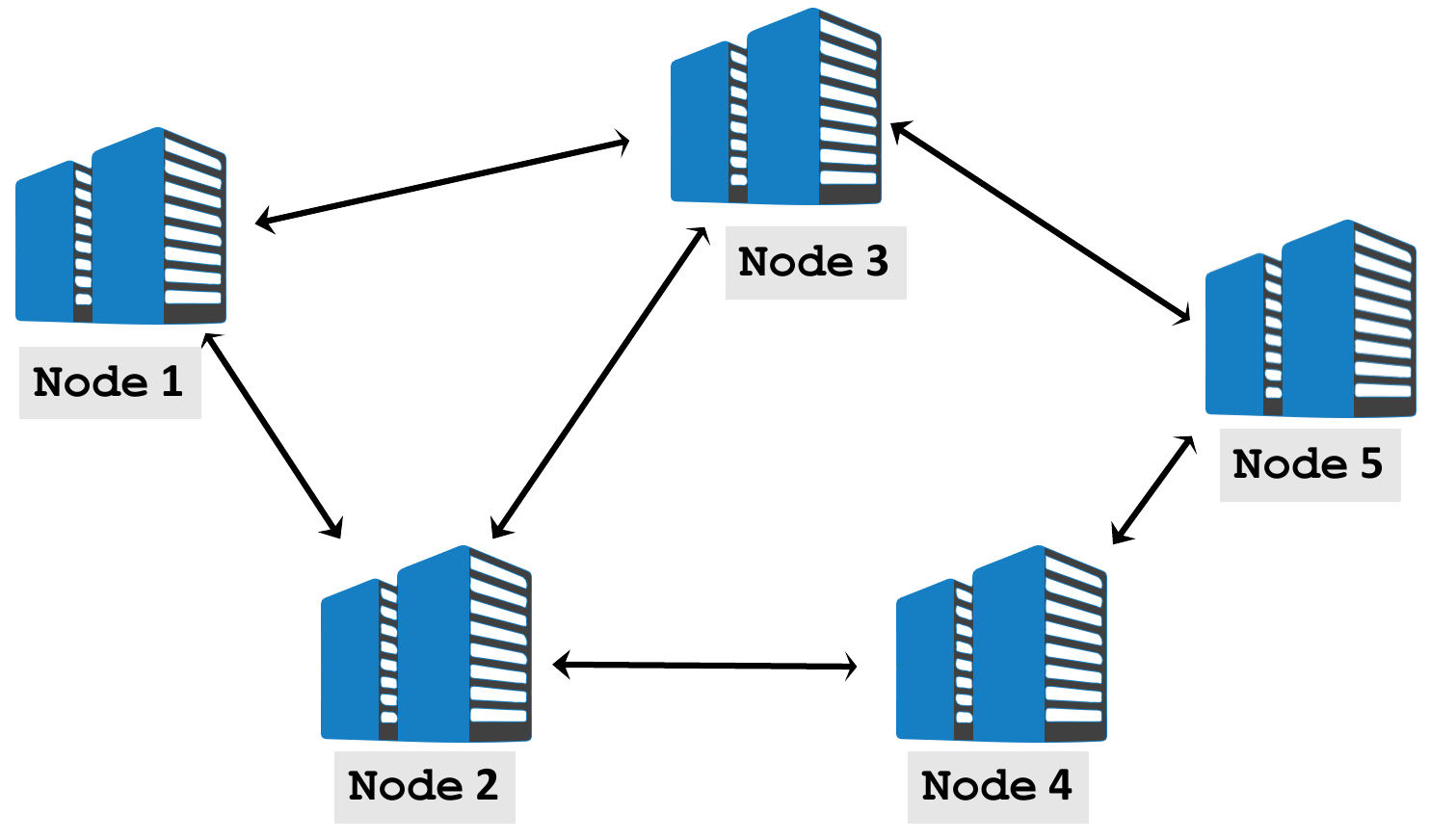}}
	\subfloat[$I_1=3$ and $I_2 =2$]{
		\includegraphics[width=0.49\textwidth] {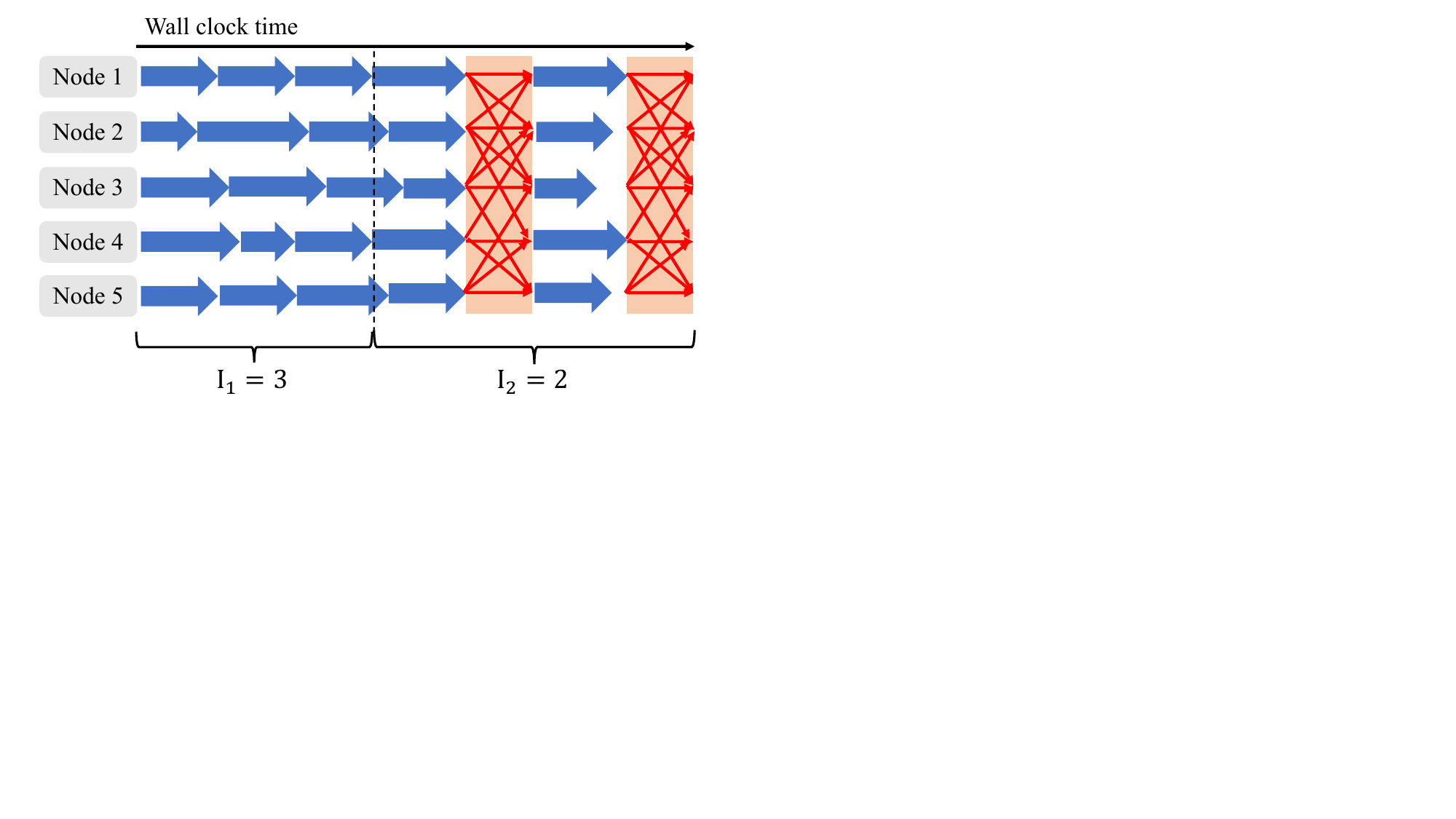}}
	\hspace{-.1in} \\
	\caption{ (a)  Illustration of decentralized system. 
		(b) Illustration of how five nodes run a round of LD-SGD with the first scheme where $I_1=3$ and $I_2 =2$. 
		Blue and red arrows respectively represent local gradient computation and communication among neighbor nodes.
	Different nodes may finish their local computation in different time due to the straggler effect.}
	\label{fig:decentralized}
\end{figure*}


In Figure~\ref{fig:decentralized} (a), we illustrate a decentralized system that doesn't have a central parameter server where each node only communicates with its neighbors. 
Conventionally, a decentralized system can be described by a graph $\GM = ([n], \W)$ where $\W$ is an $n\times n$ doubly stochastic matrix describing the weights of the edges. A nonzero entry $w_{ij}$ indicates that the $i$-th and $j$-th nodes are connected. 

\begin{defi}
	We say a matrix $\W=[w_{ij}] \in \RB^{n {\times} n}$ to be symmetric and doubly stochastic, if $\W$ is symmetric and each row of $\W$ is a probability distribution over the vertex set $[n]$, i.e., $w_{ij} \geq 0, \W = \W^\top$, and $ \W \1_n =\1_n$.
\end{defi}


\subsection{Notation}
Let $\x^{(k)} \in \RB^d$ be the optimization variable held by the $k$-th node.
The step is indicated by a subscript, e.g., $\x_t^{(k)}$ is the parameter held by the $k$-th node in step $t$.
Note that at any time moment, $\x^{(1)}, \cdots , \x^{(n)}$ may not be equal.
The concatenation of all the variables is
\begin{equation}
\X  := \left[ \x^{(1)}, \cdots, \x^{(n)} \right] \in \RB^{d {\times} n}.
\end{equation}
The averaged variable is
$ \overline{\x} := \frac{1}{n} \sum_{k=1}^n  \x^{(k)} = \frac{1}{n} \X \1_n$.
The derivative of $F_k$ w.r.t.\ $\x^{(k)}$ is $\nabla F_k (\x^{(k)};\xi^{(k)}) \in \RB^d$ and the concatenated gradient evaluated at $\X$ with datum $\xi$ is
\begin{equation*}
\G( \X; \xi) := \left[ \nabla F_1(\x^{(1)};\xi^{(1)}), \cdots, \nabla F_n(\x^{(n)};\xi^{(n)})\right] \in \RB^{d \times n}.
\end{equation*} 
We denote the set of natural numbers by $\sN = \{1,2, \cdots\}$.
We define $[n] : = \{1, 2, \cdots, n\}$ and $[s:t]$ means the interval between the positive integers $s$ and $t$, i.e., if $s\le t$, $[s:t] =  \{l \in \sN :s \le l \le t\}$, otherwise $[s:t] = \emptyset$.
For any set $\IM$ and real number $x$, we define $\IM + x := \{ t = y + x: y \in \IM  \}$.

\subsection{Decentralized SGD (D-SGD)}
D-SGD works in the following way~\cite{bianchi2013performance,lan2017communication}. 
At Step $t$, the $k$-th node randomly chooses a local datum $\xi_{t}^{(k)}$, and uses its current local variable $\x_{t}^{(k)}$ to evaluate the stochastic gradient $\nabla F_k \big( \x_{t}^{(k)} ; \xi_t^{(k)}\big)$. 
Then each node performs stochastic gradient descent (SGD) to obtain an intermediate variable $\x_{t+\frac{1}{2}}^{(k)}$ and finally finishes the update by collecting and aggregating its neighbors' intermediate variables:
\begin{eqnarray}
\x_{t+\frac{1}{2}}^{(k)} 
& \longleftarrow & \x_{t}^{(k)} - \eta \nabla F_k \big( \x_{t}^{(k)} ; \xi_t^{(k)}\big), \label{eq:psgd1}\\
\x_{t+1}^{(k)} 
& \longleftarrow &  \sum_{l \in \NM_k} w_{kl} \x_{t + \frac{1}{2}}^{(l)}, \label{eq:psgd2}
\end{eqnarray}
where $\NM_k = \{ l \in [n] \, | \, w_{kl} > 0 \}$ contains the indices of the $k$-th node's neighbors. 
In matrix form, this can be captured by
$\X_{t+1} = (\X_t - \eta \G(\X_t; \xi_t)) \W$.
On the right of the vertical imaginary line in Figure~\ref{fig:decentralized} (b), we depict two rounds of D-SGD. 
D-SGD requires $T$ communications per $T$ steps.

\begin{remark}
	The order of Step~\ref{eq:psgd1} and Step~\ref{eq:psgd2} in D-SGD can be exchanged .
	In this way, we first average the local variable with neighbors and then update the local stochastic gradient into the local variable. 
	The update rule becomes
	$\x_{t+1}^{(k)} 
	\longleftarrow   \sum_{l \in \NM_k} w_{kl} \x_{t}^{(l)} - \eta \nabla F_k \big( \x_{t}^{(k)} ; \xi_t^{(k)} \big)$.
	The benefit is that the computation of stochastic gradients (i.e., $\nabla F_k \big( \x_{t}^{(k)} ; \xi_t^{(k)} \big)$) and communication (i.e., Step~\ref{eq:psgd2}) can be run in parallel. 
	Our theory in latter section is applicable to these cases.
\end{remark}

\section{Local Decentralized SGD (LD-SGD)}
In this section, we first propose LD-SGD which is able to incorporate arbitrary update schemes.
Then we present  convergence analysis for it.

\subsection{The Algorithm}

\begin{algorithm}[h!]
	\caption{Local Decentralized SGD (LD-SGD)}
	\label{alg:main}
	\begin{algorithmic}
		\STATE {\bfseries Input:} total steps $T$, step size $\eta$, communication set $\IM_T$
		\IF {use randomized $\W$}
		\STATE{Compute the distribution $\DM$ according to~\cite{wang2019matcha}.}
		\ENDIF
		\FOR{$t=1$ {\bfseries to} $T$}
		\STATE $\X_{t+\frac{1}{2}} \gets \X_t - \eta \G(\X_t;\xi_t)$
		\IF{$  t \in \IM_T$}
		\IF {use randomized $\W$}
		\STATE{Independently generate $\W \sim \DM$}
		\ENDIF
		\STATE $\X_{t+1} \gets \X_{t+\frac{1}{2}}\W$  \RCOMMENT { Communication }
		\ELSE
		\STATE $\X_{t+1} \gets \X_{t+\frac{1}{2}}$ \RCOMMENT { Local updates }
		\ENDIF
		\ENDFOR
	\end{algorithmic}
\end{algorithm}

Algorithm~\ref{alg:main} summarizes the proposed LD-SGD algorithm.
We can write it in matrix form:
\begin{equation}
\X_{t+1} = (\X_{t} - \G(\X_t; \xi_t))\W_t,
\end{equation}
where  $\W_t \in \RB^{n \times n}$ is the connected matrix defined by
\begin{equation}
\label{eq:W_main}
\W_t = \left\{ \begin{array}{ll}
\I_n & \text{if} \ t  \notin \IM_T,  \\
\W & \text{if} \ t  \in \IM_T.
\end{array}\right. 
\end{equation}
Here $\W$ is a prespecified doubly stochastic matrix that is deterministic.
At each iteration $t$, each node first performs one step of SGD independently using data stored locally.
When $t \notin \IM_T$, each node doesn't communicate with others and goes into the next iteration with $\X_{t+1} = \X_{t} - \G(\X_t; \xi_t)$.
When $t \in \IM_T$, each node then perform one step of decentralized communication and use the resulting vector as a new parameter, i.e., $\X_{t+1} = (\X_{t} - \G(\X_t; \xi_t))\W$.
It periodically performs multiple local updates and D-SGD.
Without local updates, LD-SGD would be the standard D-SGD algorithm.
As an extension, we can use a randomized $\W$ that is generated independently at each round.
This is motivated by \cite{wang2019matcha} which proposes to use a randomized $\W$ that only activates a small portion of devices to improve communication efficiency.
For clear presentation, we use the deterministic $\W$ for our theory.
However, it is handy to parallel our theoretical results to an i.i.d. sequence of $\W$'s.

Let $\IM_T$ index the steps where decentralized SGD is performed.
Different choices of $\IM_T$ give rise to different update schemes and then lead to different communication efficiency.
For example, when we choose $\IM_T^0 = \{ t \in [T] : t \ \text{mod} \ I = 0  \} $ where $I$ is the communication interval, LD-SGD recovers the previous work PD-SGD~\cite{wang2018cooperative}.
Therefore, it is natural to explore how different $\IM_T$ affects the convergence of LD-SGD.
Our theory allows for arbitrary $\IM_T \subset [T]$.


\subsection{Convergence Analysis}
\label{sec:convergence}

\subsubsection{Assumptions}

In Eq.~\eqref{eq:goal}, we define $f_k(\x) := \EB_{\xi \sim \DM_k} \big[ F_{k} \left( \x ; \xi\right) \big]$ as the objective function of the $k$-th node.
Here, $\x $ is the optimization variable and $\xi $ is a data sample.
Note that $f_k(\x)$ captures the data distribution in the $k$-th node.
We make a standard assumption: $f_1 , \cdots , f_n$ are smooth.

\begin{assumption}[Smoothness]
	\label{asum:smooth}
	For all $k \in [n]$, $f_k$ is smooth with modulus $L$, i.e., 
	\begin{equation*}
	\big\| \nabla f_k(\x) - \nabla f_k(\y) \big\| 
	\: \leq \: L \big\| \x - \y \big\|, 
	\quad \forall \  \x, \y \in \RB^d. 
	\end{equation*}
\end{assumption}

We assume bounded stochastic gradients variance, an assumption has been made by the prior work \cite{lian2017can,wang2018cooperative,tang2018d,tang2018communication}.

\begin{assumption}[Bounded variance]
	\label{asum:within-var}
	There exists some $\sigma > 0$ such that $\forall \ k \in [n]$, 
	\begin{equation*}
	\EB_{\xi \sim \DM_k} \big\| \nabla F_k(\x;\xi) - \nabla f_k(\x) \big\|^2 
	\: \leq \: \sigma^2, 
	\quad \forall \ \x \in \RB^d.
	\end{equation*}
\end{assumption}

Recall from~\eqref{eq:goal} that $f(\x) = \frac{1}{n} \sum_{k=1}^n f_k(\x)$ is the global objective function.
If data distributions are not identical ($\DM_k \neq \DM_l$ for $k \neq l$), then the global objective is not the same to the local objectives.
In this case, we define $\kappa$ to quantify the degree of non-iid.
If the data across nodes are iid, then $\kappa = 0$.

\begin{assumption}[Degree of non-iid]
	\label{asum:inter-var}
	There exists some $\kappa \ge 0$ such that
	\begin{equation*}
	\frac{1}{n}\sum_{k=1}^n \big\| \nabla f_k(\x) - \nabla f(\x) \big\|^2
	\: \leq \: \kappa^2, 
	\quad \forall \ \x \in \RB^d.
	\end{equation*}
\end{assumption}

Finally, we need to assume the nodes are well connected;
otherwise, the update in one node cannot be propagated to another node within a few iterations.
In the worst case, if the system is not fully connected, the algorithm will not minimize $f(\x)$. 
We use $\rho = |\lambda_2|$ to quantify the connectivity where $\lambda_2$ is the second largest absolute eigenvalue of $\W$. A small $\rho$ indicates nice connectivity.
If the connection forms a complete graph, then $\W =\frac{1}{n} \1_n \1_n^\top$, and thus $\rho = 0$.

\begin{assumption}[Nice connectivity]
	\label{asum:W}
	The $n\times n$ connectivity matrix $\W $ is symmetric doubly stochastic.
	Denote its eigenvalues by $1 = |\lambda_1| > |\lambda_2| \ge \cdots \ge |\lambda_n| \ge 0$. 
	We assume the spectral gap $1- \rho \in (0, 1]$ where $\rho = |\lambda_2| \in [0, 1)$.
\end{assumption}

Actually, our theory can be directly extended to stochastic $\W$, that is, at each communication round, we generate an instance of $\widehat{\W}$ that conforms to a given distribution $\DM$ in an independent and identical manner.
For completeness, the resulting Algorithm is summarized in Algorithm~\ref{alg:main_randomW}.
The only difference for theory is to replace $\rho(\W)$ with $\rho(\EB_{\widehat{\W} \sim  \DM} \widehat{\W})$ and similar results follows.
In experiments, we test random $\W$ for LD-SGD following the same way as~\cite{wang2019matcha} did.
This is because nodes in reality are often connected via wireless connections, and hence, random matrices are more realistic.

\subsubsection{Main Results}

Recall that $\overline{\x}_t = \frac{1}{n} \sum_{k=1}^n \x_t^{(k)} $ is defined as the averaged variable in the $t$-th iteration. 
Note that the objective function $f(\x)$ is often non-convex when neural networks are applied. 
$\rho_{s, t-1}$ is very important in our theory because it captures the characteristics of each update scheme. 
All the proof can be found in~\ref{append:PD-SGD}.

Typically a single step of decentralized communication pushes all local parameters to move towards their mean, but can't make sure they are synchronized (and identical).
This implies one decentralized communication happening many iterations before will affect the current update due to such incomplete synchronization.
For example, the aggregation performed at iteration $s$ will propagate the variance of stochastic gradient computed at iteration $s$ to the current update $t$, incurring a multiplier to the final variance term ($\sigma^2$).
However, communication shrinks the effect in a exponential manner with the exponent $\rho$.
We capture the shrinkage effect caused by decentralized communication starting from iteration $s$ to $t$ by $\rho_{s, t-1}$.

\begin{defi}
	\label{def:rho}
	For any $s < t$, define $\rho_{s, t-1} = \|\Ph_{s, t-1} - \frac{1}{n} \1_n\1_n^\top \|$ where $\Ph_{s, t-1} = \prod_{l=s}^{t-1} \W_l$ with $\W_l$ given in~\eqref{eq:W_main}.
	Actually, we have $\rho_{s, t-1} = \rho^{|[s: t-1] \cap \IM_T|}$, where $[s:t-1] = \{l \in \sN: s \le l \le t-1\}$ and $\rho$ is defined in Assumption~\ref{asum:W}.
\end{defi}

\begin{theorem}[LD-SGD with any $\IM_T$]
	\label{thm:PD-SGD}
	Let Assumptions~\ref{asum:smooth},~\ref{asum:within-var},~\ref{asum:inter-var},~\ref{asum:W} hold and the constants $L$, $\kappa$, $\sigma$, and $\rho$ be defined therein. 
	Let $\Delta = f(\overline{\x}_0) - \min_{\x } f (\x )$ be the initial error.
	For any fixed $T$, 
	\begin{gather*}
	A_T =\frac{1}{T}\sum_{t=1}^{T} \sum_{s=1}^{t-1} \rho^2_{s, t-1}, 
	B_T = \frac{1}{T} \sum_{t=1}^{T} \left(\sum_{s=1}^{t-1} \rho_{s, t-1} \right)^2, \\
	C_T = \max_{s \in [T-1]} \sum_{t=s+1}^T  \rho_{s, t-1}  \left( \sum_{l=1}^{t-1}  \rho_{l, t-1} \right).
	\end{gather*}
	If the learning rate $\eta$ is small enough such that
	\begin{equation}
	\label{eq:lr}
	\eta \: < \: \min \bigg\{ \frac{1}{2L}, \; \frac{1}{4\sqrt{2}L \sqrt{C_T}} \bigg\} ,
	\end{equation}
	then
	\begin{align}
	\label{eq:main}
	\frac{1}{T} \sum_{t=1}^{T} & \EB \big\|\nabla f (\overline{\x}_t) \big\|^2  
	\: \leq \: \nonumber \\
&	{\underbrace{\frac{2\Delta}{\eta T} + \frac{\eta L\sigma^2}{n}}_{\text{fully sync SGD}}}  +
	{\underbrace{\vphantom{ \left(\frac{a^{0.3}}{b}\right) } 16\eta^2L^2 \left( A_T \sigma^2 +B_T\kappa^2 \right)}_{ \text{residual error}}} .
	\end{align}
\end{theorem}

The constant learning rate can be replaced by annealing learning rate, but the convergence rate remains the same.

\begin{corollary}
	\label{cor:computation}
	If we choose the learning rate as $\eta = \sqrt{\frac{n}{T}}$ in Theorem~\ref{thm:PD-SGD}, then when $T > 4L^2 n \max\{1,4C_T\} $, we have, 
	\begin{equation}
	\label{eq:linear}
	\frac{1}{T} \sum_{t=1}^{T} \EB \|\nabla f (\overline{\x}_t ) \|^2 
	\leq
	\frac{2 \Delta {+} L \sigma^2 }{\sqrt{n T}}
	{+} \frac{ 16nL^2  \left( A_T \sigma^2 {+} B_T\kappa^2 \right) }{ T }.
	\end{equation}
\end{corollary}

\subsection{Sufficient condition for convergence.} 
If the chosen $\IM_T$ satisfies the sublinear condition that 
\begin{equation}
\label{eq:sublinear}
A_T = o(T), \ B_T = o(T) \ \text{and} \ C_T = o(T),
\end{equation}
we thereby prove the convergence of LD-SGD with the update scheme $\IM_T$ to a stationary point asymptotically, e.g., a local minimum or saddle point (which follows from Corollary~\ref{cor:computation}).
However, not every update scheme satisfies~\eqref{eq:sublinear}.
For example, when $\IM_T = \{T\}$, we have $\rho_{s, t-1}=1$ for all $s < t \le T$ and thus $A_T = \Theta(T), B_T = \Theta(T^2)$ and $C_T = \Theta(T^2)$.
But Theorem~\ref{thm:bound} shows that as long as $\text{gap}(\IM_T)$ is small enough (for example, $\text{gap}(\IM_T) = O(T^a)$ for some $a \in [0, 1/2)$), the sublinear condition holds.
The gap indicates the largest number of local SGD steps before a communication round of D-SGD is triggered.
Intuitively, frequent communication results a small gap.
So there is still a wide range of $\IM_T$ that meets the condition as long as we control the gap.

\begin{defi}[Gap]
	\label{def:gap}
	For any set $\IM_T = \{ e_1, \cdots, e_g \} \subset [T]$ with $e_i < e_{j}$ for $i < j$, the gap of $\IM_T$ is defined as
	\begin{equation}
	\text{gap}(\IM_T) = \max_{i \in [g+1]} (e_i - e_{i-1}) 
	\end{equation}
	where  $e_0 =0, e_{g+1}=T$.
\end{defi}

\begin{theorem}
	\label{thm:bound}
	Let $A_T, B_T, C_T$ be defined in Theorem~\ref{thm:PD-SGD}. Then for any $\IM_T$, we have
	\begin{gather*}
		A_T \le   \frac{ \text{gap}(\IM_T)}{2} \left[ \frac{1+\rho^2}{1-\rho^2} -1  \right],
	\max\{ B_T, C_T \} \le  \frac{\text{gap}(\IM_T)^2}{(1-\rho)^2}. 
	\end{gather*}
\end{theorem}

\section{Two Proposed Update Schemes}
\label{sec:algo}
Before we move to the discussion of our results, we  first specify two classes of update schemes, both of which satisfy the sublinear condition~\eqref{eq:sublinear}.
The proposed update schemes  also deepen our understandings of the main results.

\subsection{Adding Multiple Decentralized SGDs}
One centralized average can synchronize all local models, while it often needs multiple decentralized communication to achieve global consensus.
Therefore, it is natural to introduce multiple decentralized SGDs (D-SGD) to $\IM_T^0$.
In particular, we set 
\begin{equation}
\IM_T^1 = \{ t \in [T]: t \ \text{mod} \ (I_1 + I_2)  \notin [I_1]  \}, 
\end{equation}
where $I_1, I_2$ are parameters that respectively control the length of local updates and D-SGD.

Therefore, in a single round of LD-SGD, nodes perform local updates only for a given number of “sub-rounds”, then preform local updates and communication for a given number of other “sub-rounds”. 
In particular, each worker node periodically alternates between two phases in a single round.
In the first phase, each node locally runs $I_1 \ (I_1 \geq 0)$ steps of SGD in parallel.\footnote{That is to perform~\eqref{eq:psgd1} for $I_1$ times.}
In the second phase, each worker node runs $I_2 \ (I_2 \geq 1)$ steps of D-SGD.
As mentioned, D-SGD is a combination of~\eqref{eq:psgd1} and~\eqref{eq:psgd2}. 
So communication only happens in the second phase; a worker node performs $\frac{I_2}{I_1 + I_2} T $ communication per $T$ steps. 
Figure~\ref{fig:decentralized} (b) illustrates one round of LD-SGD with $\IM_T^1$ when $I_1=3$ and $I_2=2$.
When LD-SGD is equipped with $\IM_T^1$, the corresponding $A_T, B_T, C_T$ are $O(1)$ w.r.t.\ $T$.
The proof is provided in~\ref{append:scheme1}.

\begin{theorem}[LD-SGD with $\IM_T^1$]
	\label{thm:scheme1}
	When we set $\IM_T = \IM_T^1$ for PD-SGD, under the same setting, Theorem~\ref{thm:PD-SGD} holds with
	\begin{gather}
	A_T \le \frac{1}{2I}\left( \frac{1+\rho^{2I_2}}{1-\rho^{2I_2}} I_1^2 + \frac{1+\rho^2}{1-\rho^2} I_1 \right)  +  \frac{\rho^2}{1-\rho^2}, \\
	\max \left\{  B_T, C_T \right\}\le K^2, K= \frac{I_1}{1 - \rho^{I_2}} + \frac{\rho}{1-\rho}.
	\end{gather}
	Therefore, LD-SGD converges with $\IM_T^1$.
\end{theorem}

The introduction of $I_2$ extends the scope of previous framework: Cooperative SGD~\cite{wang2018cooperative}.
As a result, many existing algorithms become special cases when the period lengths $I_1, I_2$, and the connected matrix $\W$ are carefully determined. 
As an evident example, we recover D-SGD by setting $I_1 = 0$ and $I_2 > 0$\footnote{We make a convention that $[0] = \emptyset$, so in this case $\IM_T^1 = [T]$.} and the conventional PD-SGD by setting $I_1 > 0$ and $I_2 = 1$.
Another important example is Local SGD (or FedAvg) that periodically averages local model parameters in a centralized manner~\cite{zhou2017convergence,lin2018don,stich2018local,yu2019parallel}.
Local SGD is a case with $I_1 > 1$, $I_2 = 1$ and $\W = \frac{1}{n} \1_n \1_n^\top$. 
We summarize examples and the comparison with their convergence results in~\ref{appen:discussion}. 


\subsection{Decaying the Length of Local Updates}


Typically, larger local computation ratio (i.e., $I_1/I_2$) incurs a higher final error, while lower local computation ratio enjoys a smaller final error but sacrifices the convergence speed.
A related phenomena is observed by an independent work~\cite{wang2018adaptive}, which finds that a faster initial drop of global loss often accompanies a higher final error.

To decay the final error, we are inspired to decay $I_1$ every $M$ rounds until $I_1$ vanishes.
In this way, we use $(I_1, I_2)$ for a first $M$ rounds, then use $(\lfloor I_1/2 \rfloor, I_2)$ for a second $M$ rounds, then use $(\lfloor I_1/2^2 \rfloor, I_2)$ for a third $M$ rounds...
We repeat this process until the $J = \lceil \log_2 I_1 \rceil $ phase where we have $\lfloor I_1/2^J \rfloor = 0$.
To give a mathematical formulation of such $\IM_T$, we need an ancillary set
\begin{equation*}
\IM(I_1, I_2, M) = \{t \in [M(I_1+I_2)]: t \ \text{mod} \ (I_1+I_2) \notin [I_1] \}, 
\end{equation*}
and then recursively define $\JM_0 = \IM(I_1, I_2, M)$ and 
\begin{equation*}
\JM_j =\IM\left( \bigg\lfloor \frac{I_1}{2^j} \bigg\rfloor , I_2, M\right) + \max(\JM_{j-1}), 1 \le j \le J,
\end{equation*}
where $\max(\JM_{j-1})$ returns the maximum number collected in $\JM_{j-1}$ and $J = \lceil \log_2 I_1 \rceil $.
Finally we set
\begin{equation}
\IM_T^2 =\cup_{j=0}^J \JM_{j} \cup [\max(\JM_{J}):T]. 
\end{equation}
The idea is simple but the formulation is a little bit complicated.
From the recursive definition, once $t \ge \max(\JM_{J})$, $I_1$ is reduced to zero and LD-SGD is reduced to D-SGD. 
When LD-SGD is equipped with $\IM_T^2$, the corresponding $A_T, B_T, C_T$ are $O(1)$ w.r.t. $T$.
The proof is provided in~\ref{append:scheme2}.

\begin{theorem}[LD-SGD with $\IM_T^2$]
	\label{thm:scheme2}
	When we set $\IM_T = \IM_T^2$ for LD-SGD, under the same setting, for $T \ge \max(\JM_{J})$, Theorem~\ref{thm:PD-SGD} holds with
	\begin{gather*}
	A_T \le\frac{1}{T} \frac{I_1}{1{-}\rho^{2I_2}} \rho^{2(T {-}\max(\JM_{J}))} + (1{-}\frac{\max(\JM_{J})}{T}) \frac{\rho^2}{1{-}\rho^2}, \\
	B_T \le K \left[\frac{1}{T} \frac{I_1}{1{-}\rho^{I_2}} \rho^{T{-}\max(\JM_{J})} + (1{-}\frac{\max(\JM_{J})}{T}) \frac{\rho}{1-\rho}\right],\\
	C_T \le K^2,
	\end{gather*}
	where $K$ is the same in Theorem~\ref{thm:scheme1}.
	Therefore, LD-SGD converges with $\IM_T^2$.
\end{theorem}

From experiments in Section~\ref{sec:exp}, the simple strategy empirically performs better than the PD-SGD.

\section{Discussion}
In this section, we will discuss some aspects of our main results (Theorem~\ref{thm:PD-SGD}) and shed light on advantages of proposed update schemes.

\subsection{Error decomposition} 
From Theorem~\ref{thm:PD-SGD}, the upper bound~\eqref{eq:main} is decomposed into two parts. 
The first part is exactly the same as the optimization error bound in parallel SGD~\cite{bottou2018optimization}. 
The second part is termed as residual errors as it results from performing periodic local updates and reducing inter-node communication. 
In previous literature, the application of local updates inevitably results the residual error~\cite{lan2017communication,stich2018local,wang2018cooperative,haddadpour19a,li2019convergence,yu2019linear}.

To go a step further towards the residual error, take LD-SGD with $\IM_T^1$ for example.
From Theorem~\ref{thm:scheme1}, the residual error often grows with the length of local updates $I = I_1 + I_2$.
When data are independently and identical distributed
\footnote{This is also possible if all nodes have access to the entire data, e.g., the distributed system may shuffle data regularly so that each node actually optimizes the same loss function.} (i.e., $\kappa = 0$), \cite{wang2018cooperative} shows that the residual error of the conventional PD-SGD grows only linearly in $I$.
~\cite{haddadpour19a} achieves the similar linear dependence on $I$ but only requires each node draws samples from its local partitions.
When data are not identically distributed (i.e., $\kappa$ is strictly positive), both~\cite{yu2019parallel} and~\cite{zhou2017convergence} show that the residual error of Local SGD grows quadratically in $I$. 
Theorem~\ref{thm:scheme1} shows that the residual error of LD-SGD with $\IM_T^1$ is $O(I\sigma^2 + I^2 \kappa^2)$, where the linear dependence comes from the stochastic gradients and the quadratic dependence results from the heterogeneity.
The similar dependence is also established for centralized momentum SGD in~\cite{yu2019linear}.

\subsection{On Linear Speedup}
Assume $\IM_T$ satisfies
\begin{equation}
\label{eq:speedup_condition}
A_T = O(\sqrt{T}) \ B_T = O(\sqrt{T}) \ \text{and} \ C_T = o(T).
\end{equation}
Note that Condition~\eqref{eq:speedup_condition} is sufficient for the sublinear condition~\eqref{eq:sublinear}.
From Corollary~\ref{cor:computation}, the convergence of LD-SGD with $\IM_T$ will be dominated by the first term $O(\frac{1}{\sqrt{nT}})$, when the total step $T$ is sufficiently large.
So LD-SGD with $\IM_T$ can achieve a linear speedup in terms of the number of worker nodes.
Both of $\IM_T^1$ and $\IM_T^2$ satisfy Condition~\eqref{eq:speedup_condition}.
Taking LD-SGD with $\IM_T^1$ for example, we have
\begin{corollary} 
	\label{cor:communication}
	In the setting of Theorem~\ref{thm:scheme1}, if we set $\eta =\sqrt{\frac{n}{T}}$ and choose $I_1, I_2$ to satisfy that $nK^2/T = 1/\sqrt{nT}$ then the bound of $\frac{1}{T} \sum_{t=1}^{T} \EB \big\|\nabla f (\overline{\x}_t) \big\|^2$ becomes
	\begin{equation*}
	\frac{2\Delta + L \sigma^2 + 4L^2(\sigma^2 + \kappa^2)}{\sqrt{nT}}.
	\end{equation*}
\end{corollary}

However, if $\IM_T$ fails to meet~\eqref{eq:speedup_condition}, the second term $O(\frac{A_T+B_T}{T} \cdot n)$ will dominate.
As a result, more worker nodes may lead to slow convergence or even divergence. 
As suggested by Theorem~\ref{thm:bound}, one way to make the first term dominate is to involve more communication.\footnote{By pigeonhole principle, we have $\text{gap}(\IM_T) \ge \frac{T}{|\IM_T|+1}$. In order to reduce $\text{gap}(\IM_T)$, one must perform more communication.}

\subsection{Communication Efficiency}
LD-SGD with $\IM_T$ needs only $|\IM_T|$ communications per $T$ total steps.
To increase communication efficiency, we are motivated to reduce the size of $\IM_T$ as much as possible.
However, as suggested by Theorem~\ref{thm:bound}, to guarantee convergence, we are required to make sure $\IM_T$ is sufficiently large (so that $\text{gap}(\IM_T)$ will be small enough).
The trade-off between communication and convergence needs a careful design of update schemes.
The two proposed update schemes have their own way of balancing the trade-off with a bounded $\text{gap}(\IM_T)$. 

For LD-SGD with $\IM_T^1$, it only needs $\frac{I_2}{I} T$ communications per $T$ total steps where $I = I_1 + I_2$.
Similar to Local SGD which has $O(T^{\frac{3}{4}} n^{\frac{3}{4}})$ communication complexity in centralized settings~\cite{yu2019parallel}, LD-SGD with $\IM_T^1$ also achieves that level.
This follows by noting from Corollary~\ref{cor:communication}, to ensure $nK^2/T = 1/\sqrt{nT}$, we have $I_1 =  (T^{\frac{1}{4}} n^{-\frac{3}{4}} - \frac{\rho}{1-\rho})(1-\rho^{I_2})$.
Then the communication complexity of LD-SGD with $\IM_T^1$ is 
\begin{equation*}
\frac{I_2}{I_2 + (T^{\frac{1}{4}} n^{-\frac{3}{4}} - \frac{\rho}{1-\rho})(1-\rho^{I_2})} T = O(T^{\frac{3}{4}} n^{\frac{3}{4}}),
\end{equation*}
which is an increasing function of $I_2$ (which follows since $(1-\rho^{I_2})/I_2$ decreases in $I_2$ and $T^{\frac{1}{4}} n^{-\frac{3}{4}} > \frac{\rho}{1-\rho}$ for large enough $T$).
Hence, large $I_2$ increase communication cost.
On the other hand, a large $I_2$ fastens convergence (since all bounds for $A_T, B_T, C_T$ in Theorem~\ref{thm:scheme1} are decreasing in $I_2$).
Therefore, $I_2$ helps trade-off convergence and communication (with a fixed $I_1$).
In experiments, larger $I_2$ often results a smaller training loss, a higher test accuracy and a higher communication cost.
The introduction of $I_2$ allows more flexibility to balance the trade-off between communication and convergence.

For LD-SGD with $\IM_T^2$, it has much faster convergence rate since from Theorem~\ref{thm:scheme2}, the bounds of $A_T$ and $B_T$ are much smaller than those of $\IM_T^1$.
However, it needs $MI_2J + (T-\max(\JM_J))$ communications per $T$ total steps, which is more than that of $\IM_T^1$ but less than that of D-SGD.
Therefore, $\IM_T^2$ can be viewed as an intermediate state between $\IM_T^1$ and D-SGD.
When $I_1$ gradually vanishes, we expect that the residual error will be reduced and better performance will follows, even though larger $|\IM_T^2|$ may increase a little bit communication cost.
From experiments in Section~\ref{sec:exp}, $\IM_T^2$ empirically has less training loss and obtains higher test accuracy than the non-decayed LD-SGD.

\subsection{Effect of connectivity $\rho$}
The connectivity is measured by $\rho$, the second largest absolute eigenvalue of $\W$.
The network connectivity $\rho$ has impact on the convergence rate via $\rho_{s, t-1}$.
Each update scheme corresponds to one way that $\rho_{s, t-1}$ depends on $\rho$.
Generally speaking, well-connectivity helps reduce residual errors and thus speed up convergence.
If the graph is nicely connected, in which case $\rho$ is close to zero, then the update in one node will be propagated to all the other nodes very soon, and the convergence is thereby fast.
As a result, the bounds in Theorem~\ref{thm:bound} are much smaller due to $\rho \approx 0$.
On the other hand, if the network connection is very sparse (in which case $\rho \approx 1$), $\rho$ will greatly slows convergence.
Take LD-SGD with $\IM_T^1$ for example.
When $\rho \approx 1$, from Theorem~\ref{thm:scheme1}, the bound of $A_T \approx \frac{1}{1-\rho} \frac{I}{2I_2}$ and the bound of $B_T \approx (\frac{1}{1-\rho} \frac{I}{I_2})^2$, both of which can be extremely large.
Therefore, it needs more steps to converge.


\section{Experiments}
\label{sec:exp}

We evaluate LD-SGD with two proposed update schemes ($\IM_T^1$ and $\IM_T^2$) 
on two tasks, namely (1) image classification on CIFAR-10 and CIFAR-100; and (2) Language modeling on Penn Treebank corpus (PTB) dataset.
All training datasets are evenly partitioned over a network of workers.
We will investigate (i) the effect of different $(I_1, I_2)$; (ii) the effect of data heterogeneity; (iii) the effect of connected topology (different $\rho$); (iv)  the effect of i.i.d. generated $\W$.
We run LD-SGD in a sufficient number of rounds that guarantees the convergence of all algorithms.
LD-SGD with $(I_1, I_2) = (0, 1)$ is the PD-SGD.
A detailed description of the training configurations is provided in Appendix~\ref{appen:exp_detail}.

\paragraph{Different $I_1/I_2$}
We evaluate the performance of LD-SGD with various $(I_1, I_2)$ on the three datasets.
The first column of Figure~\ref{fig:iid} shows training loss v.s. epoch.
Since the learning rate is delayed twice for image classification tasks, there is two sudden jumps in losses for curves therein.
Typically, the larger $I_1/I_2$, the larger final loss error.
This is because, when $I_1/I_2$ is large, LD-SGD performs a relatively many number of local computation, which would accumulate a large residual error according to our theory and thus make the loss error larger than that of small $I_1/I_2$.
However, when we study the optimization through the len of running time (that is the sum of computation time and communication time), we will find that local computation is useful in fastening realtime convergence and sometimes does help in better accuracy (see Figure~\ref{fig:iid-test-acc} in the Appendix).
Indeed, larger $I_1/I_2$ completes the training more earlier without degrading the accuracy too much.
The main reason is that communication is more time-consuming than local computation, while both of them could help convergence.
Once we trade communication for more local computation, we could achieve both faster convergence and good accuracy.

\paragraph{Same $I_1/I_2$}
Then, we fix the ratio of $I_1/I_2$ and investigate the performance of different realizations of $(I_1, I_2)$.
We show the result of loss v.s. running time in the third column of Figure~\ref{fig:iid}.
It seems that with the ratio $I_1/I_2$ fixed, the real time convergence remains almost unchanged.
In particular, all curves of PTB results collapse to one single curve.
It implies the ratio $I_1/I_2$ actually controls the communication efficiency.

\begin{figure*}[t]
	\centering
	\vspace{-0.15in} 
	\subfloat[Loss v.s. epoch on CIFAR 10]{
	\includegraphics[width=0.25\textwidth] {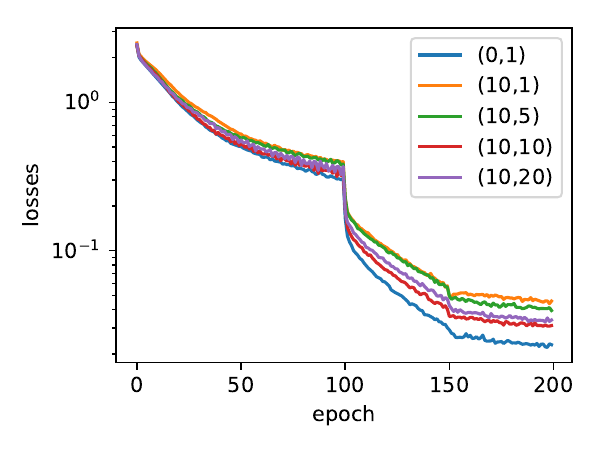}}
\hspace{-0.15in} 
\subfloat[Test acc v.s. time on CIFAR 10]{
	\includegraphics[width=0.25\textwidth] {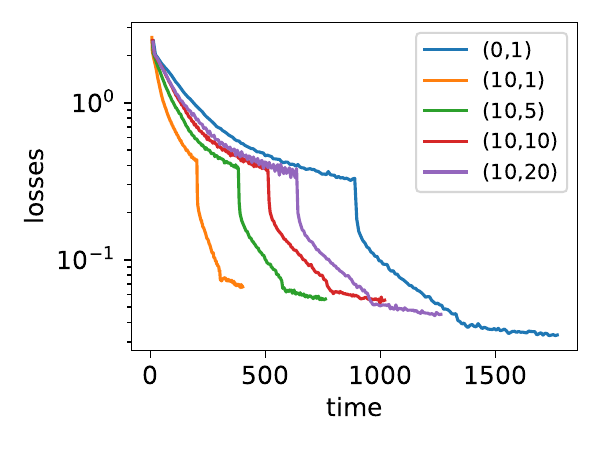}}
\hspace{-0.15in} 
	\subfloat[Same $I_1/I_2$  on CIFAR 10]{
	\includegraphics[width=0.25\textwidth] {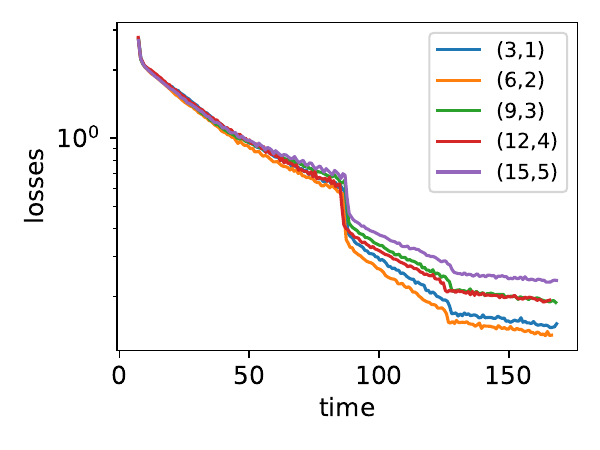}}
\hspace{-0.15in} 
	\subfloat[Decay strategy on CIFAR 10]{
	\includegraphics[width=0.25\textwidth] {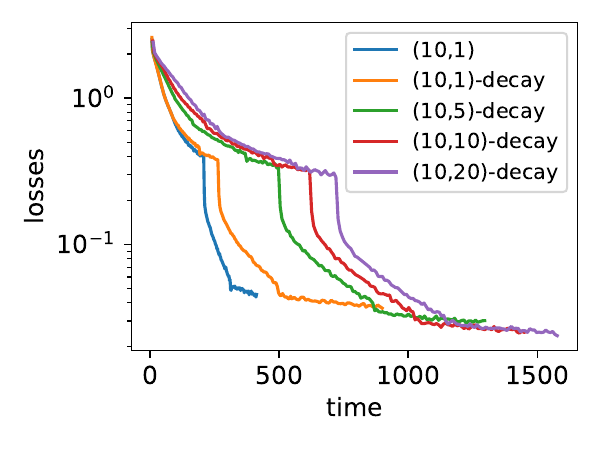}}\\
	\vspace{-0.15in} 
\subfloat[Loss v.s. epoch on CIFAR 100]{
	\includegraphics[width=0.25\textwidth] {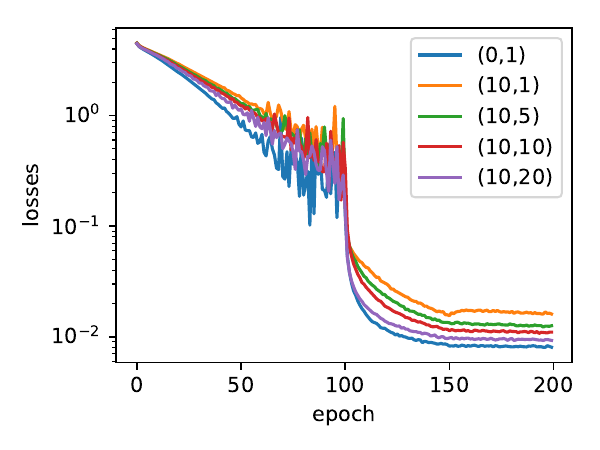}}
	\hspace{-0.15in} 
\subfloat[Test acc v.s. time on CIFAR 100]{
	\includegraphics[width=0.25\textwidth] {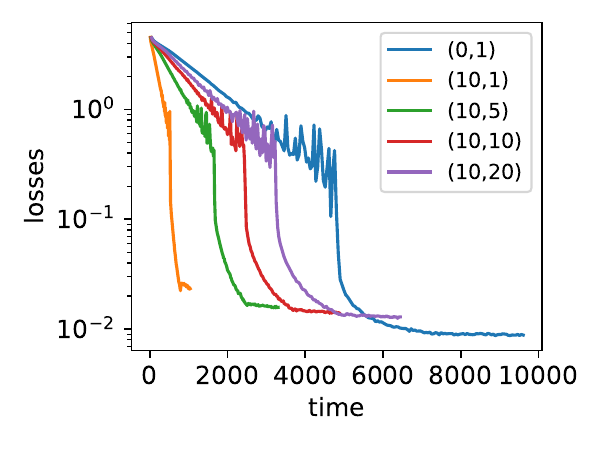}}
\hspace{-0.15in} 
\subfloat[Same $I_1/I_2$ on CIFAR 100]{
	\includegraphics[width=0.25\textwidth] {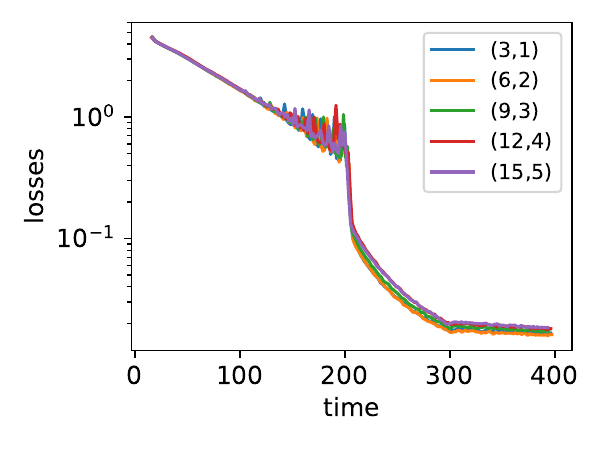}}
\hspace{-0.15in} 
\subfloat[Decay strategy on CIFAR 100]{
	\includegraphics[width=0.25\textwidth] {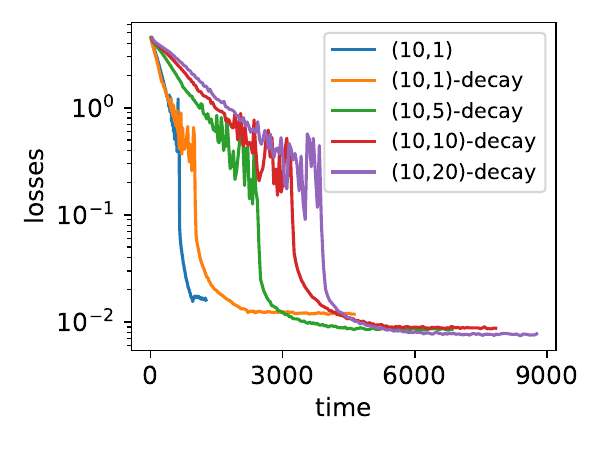}}\\
	\vspace{-0.15in} 
\subfloat[Loss v.s. epoch on PTD]{
	\includegraphics[width=0.25\textwidth] {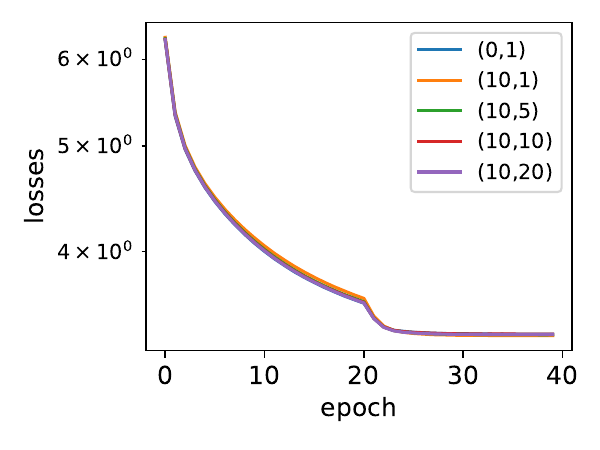}}
	\hspace{-0.15in} 
	\subfloat[Test acc v.s. time on PTD]{
		\includegraphics[width=0.25\textwidth] {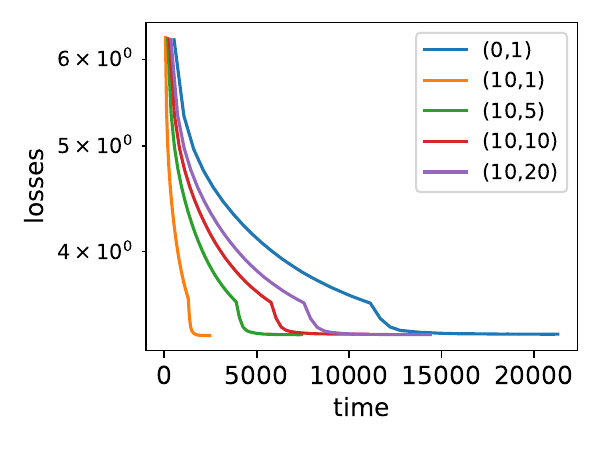}}
\hspace{-0.15in} 
\subfloat[Same $I_1/I_2$ on PTD]{
	\includegraphics[width=0.25\textwidth] {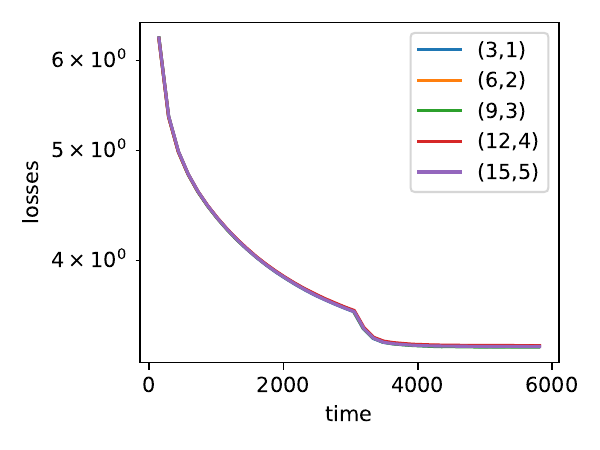}}
\hspace{-0.15in} 
\subfloat[Decay strategy on PTD]{
	\includegraphics[width=0.25\textwidth] {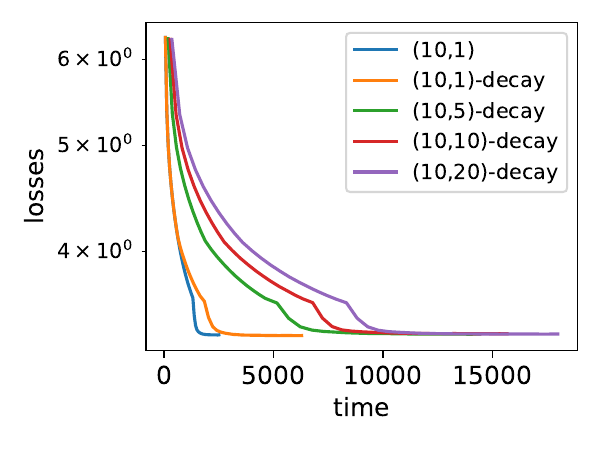}}\\
	\caption{Loss comparison on LD-SGD with different $(I_1, I_2)$. The first two columns show the results of different $I_1/I_2$, the third column shows those of the same $I_1/I_2$, and the final column shows those of the decay strategy.
	For test accuracy comparison, one can refer to Figure~\ref{fig:iid-test-acc} in Appendix.
}
	\label{fig:iid}
	\vspace{-0.15in} 
\end{figure*}

\paragraph{The decay strategy.}
We observe that larger $I_1/I_2$ often incurs a large final error. 
It is intuitive to gradually decay $I_1/I_2$ until $I_1$ reaches zero, in order to have a smaller error.
In experiments, we heuristically halve $I_1$ every $M=40$ epochs, i.e., $I_1 = \lfloor  I_1 / 2 \rfloor$. 
The result of the decay strategy is shown in the rightmost column of Figure~\ref{fig:iid}.
We can see that LD-SGD with the decay strategy typically has smaller final training loss that the non-decayed counterparts, though at the price of a little bit communication efficiency.
Clearly the decay strategy does lower the final error and even improve the test accuracy.

\paragraph{Data heterogeneity}
In previous experiments, we distributed data evenly and randomly to ensure each node has (almost) iid data.
We then distribute all samples in a non-i.i.d. manner and want to explore whether and how data heterogeneity slows down convergence rate.
It means $\kappa$ in our theory could be very large.
For image classification tasks, we evenly distribute different classes of images so that each node will only have samples from a same number of specific classes.\footnote{After this procedure, if we has unassigned classes, we then distributed these samples from those classes evenly and randomly into all nodes.}
Therefore, the classes are not assigned uniformly randomly and the training data on each node is skewed.
Since language itself has heterogeneity, for language modeling, we just divide the whole dataset evenly into different nodes instead of giving each node a copy of the whole dataset.

The result is shown in Figure~\ref{fig:niid}.
For a given choice of $(I_1, I_2)$, we have the following observations.
First, non-i.i.d. dataset often makes the training curves have larger fluctuation (see (b)).
Second, LD-SGD converges slightly faster on i.i.d. data than on non-i.i.d. one in real time measurement (see (a), (b) and (e)).
This can be explained by our theory.
Data heterogeneity enlarge the quantity $\kappa$ and slows down convergence from the main theorems.
Third, models trained from i.i.d. datasets have slightly better generalization since it obtains slightly higher test accuracy (see (c) and (d)).
Indeed, non-iid data makes the training task harder and sacrifices generalization a little bit.
Finally, non-i.i.d. dataset often results in larger final training errors (see (a), (b) and (e)).

\begin{figure*}[ht]
	\centering
	\vspace{-0.1in} 
	\subfloat[Loss on CIFAR 10]{
		\includegraphics[width=0.2\textwidth] {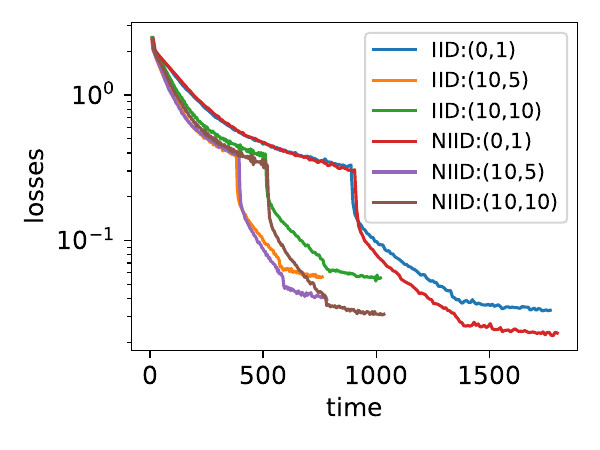}}
	\hspace{-0.17in} 
	\subfloat[Loss on CIFAR 100]{
		\includegraphics[width=0.2\textwidth] {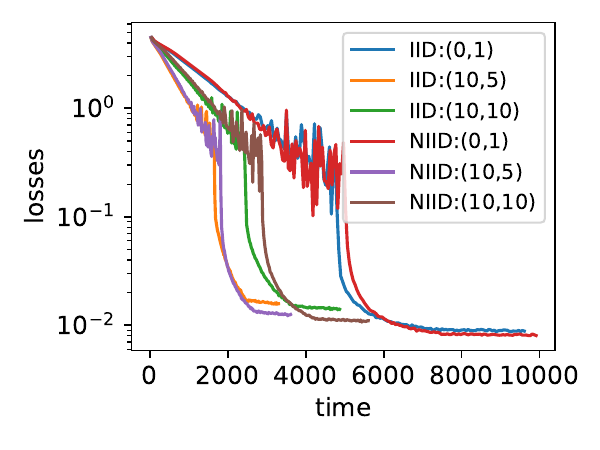}}
	\hspace{-0.17in} 
	\subfloat[Test acc on CIFAR 10]{
		\includegraphics[width=0.20\textwidth] {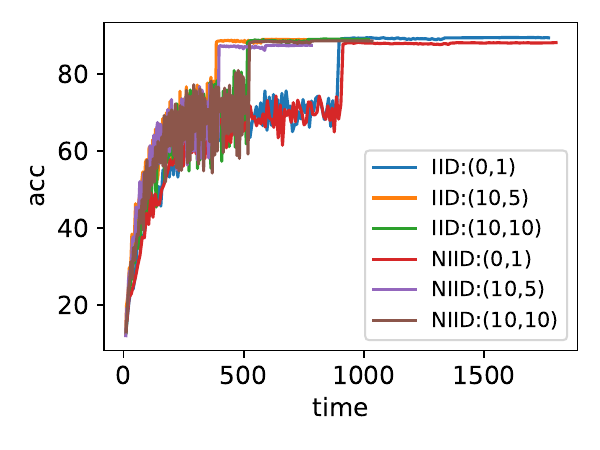}}
	\hspace{-0.17in} 
	\subfloat[Test acc on CIFAR 10]{
		\includegraphics[width=0.20\textwidth] {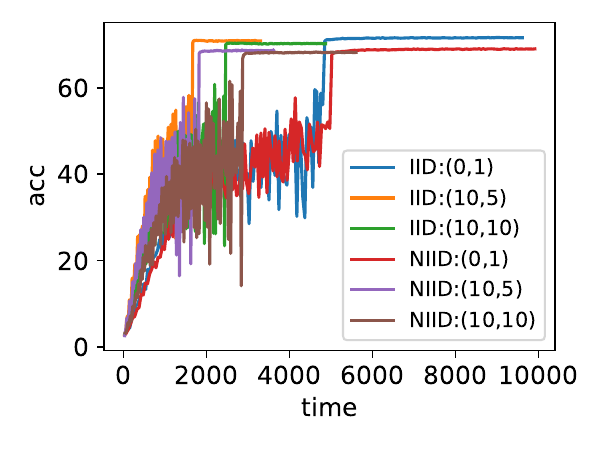}}
	\hspace{-0.17in} 
	\subfloat[Loss on PTD]{
		\includegraphics[width=0.20\textwidth] {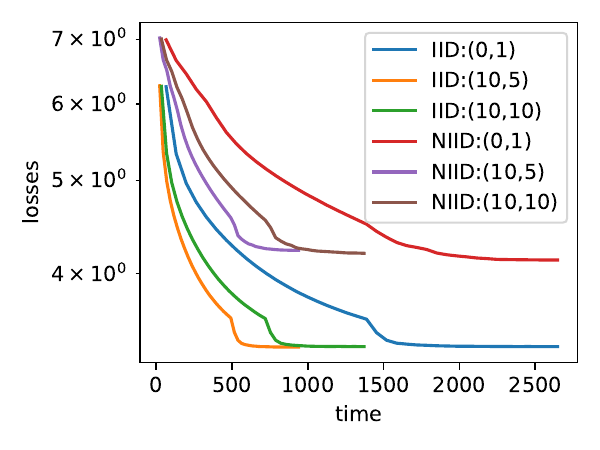}}
	\caption{The effect of data heterogeneity. 
	Non-iid data typically slows convergence rate, increases the final training error and sacrifices a little bit generalization.}
	\label{fig:niid}
	\vspace{-0.1in} 
\end{figure*}


\paragraph{Different topology}
In our theory, the topology affects the convergence rate through $\rho$.
Smaller $\rho$ will have smaller residual errors and faster convergence rate.
We test two graphs that both have $n=16$ nodes (see Figure~\ref{fig:graph} in the Appendix for illustration); for Graph 1, $\rho = 0.964$ while for Graph 2, $\rho = 0.693$.
With the results in Figure~\ref{fig:rho}, we find that LD-SGD indeed converges slightly faster and has a higher test accuracy on Graph 2.
However, in the task of Language modeling on PTD dataset, the convergence behaviors on the two graphs have little difference (see (e) in the Figure~\ref{fig:rho}).
We speculate this is because language modeling is typically harder than image classification and is not sensitive to the underlying topology.

\begin{figure*}[ht]
	\centering
	\vspace{-0.1in} 
	\subfloat[Loss on CIFAR 10]{
	\includegraphics[width=0.2\textwidth] {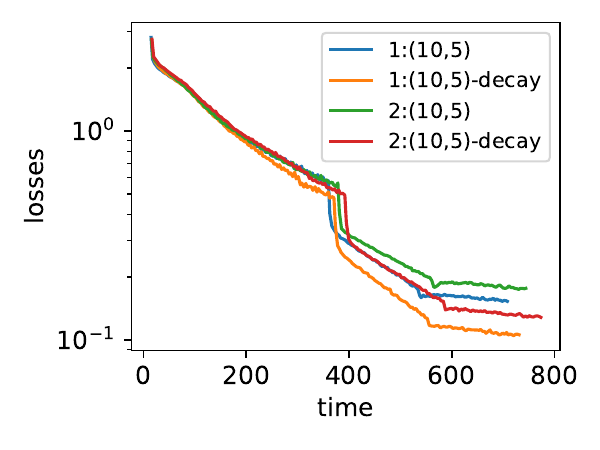}}
	\hspace{-0.17in} 
	\subfloat[Loss on CIFAR 100]{
	\includegraphics[width=0.2\textwidth] {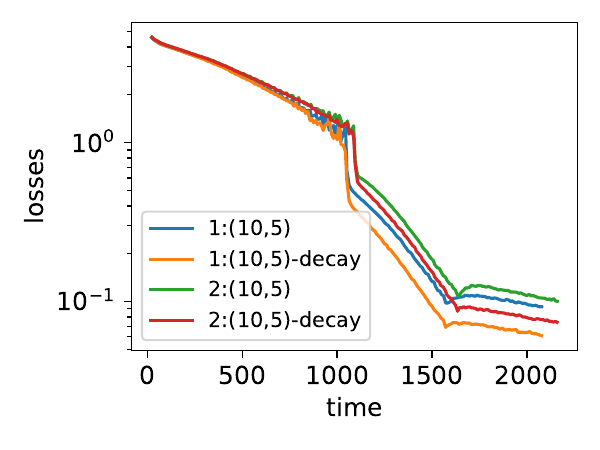}}
	\hspace{-0.17in} 
	\subfloat[Test acc on CIFAR 10]{
	\includegraphics[width=0.20\textwidth] {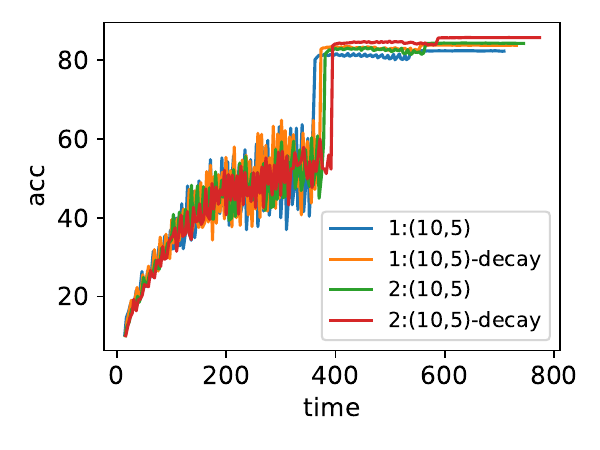}}
	\hspace{-0.17in} 
	\subfloat[Test acc on CIFAR 10]{
	\includegraphics[width=0.20\textwidth] {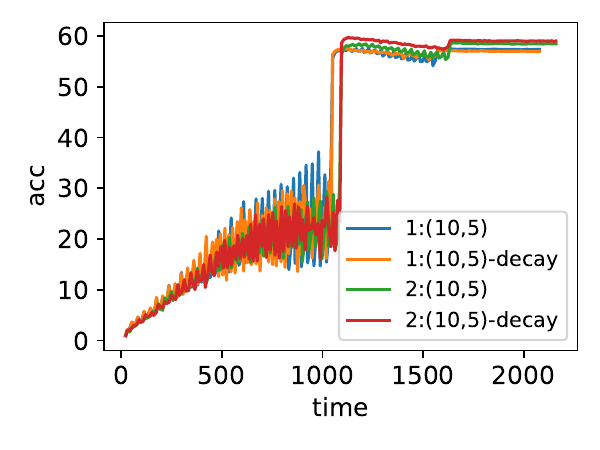}}
	\hspace{-0.17in} 
	\subfloat[Loss on PTD]{
		\includegraphics[width=0.20\textwidth] {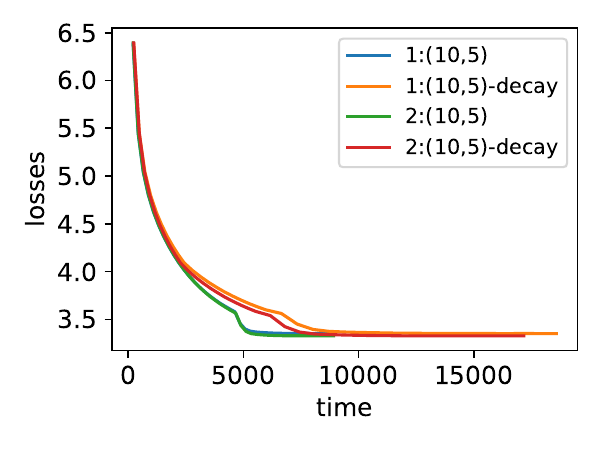}}
	\caption{Effect of different topology graphs.
		We consider two graphs depicted in Figure~\ref{fig:graph}, both of which have 16 nodes and same data allocation for fair comparison.
	(a) and (b) show training loss v.s. real time, while (c) and (d) show test accuracy v.s. real time.}
	\label{fig:rho}
	\vspace{-0.1in} 
\end{figure*}

\begin{figure*}[h!]
	\centering
	\vspace{-0.1in} 
	\subfloat[Training loss on CIFAR 10]{
		\includegraphics[width=0.3\textwidth] {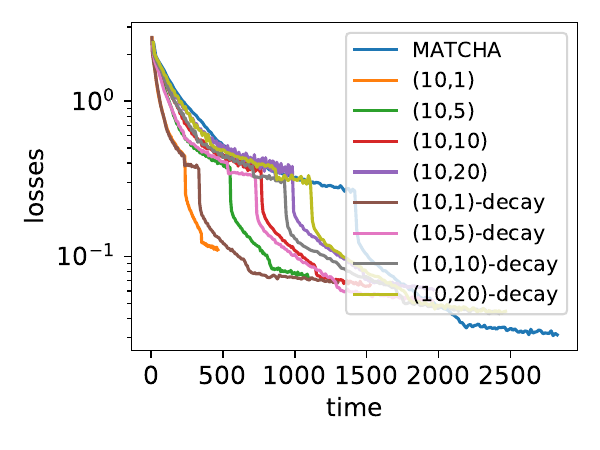}}
	\hspace{-0.15in} 
	\subfloat[Training loss on CIFAR 100]{
		\includegraphics[width=0.3\textwidth] {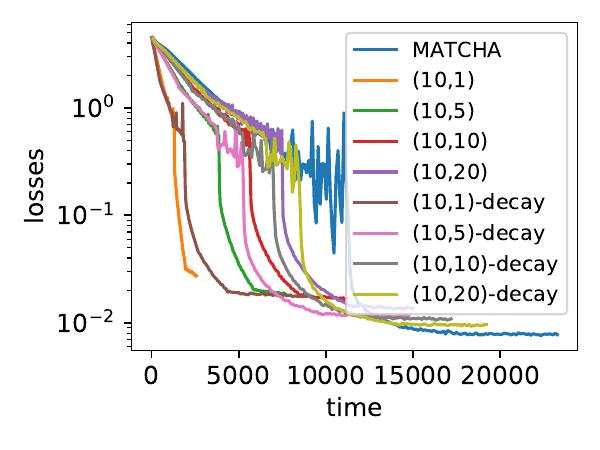}}
	\hspace{-0.15in} 
	\subfloat[Training loss on PTD]{
		\includegraphics[width=0.3\textwidth] {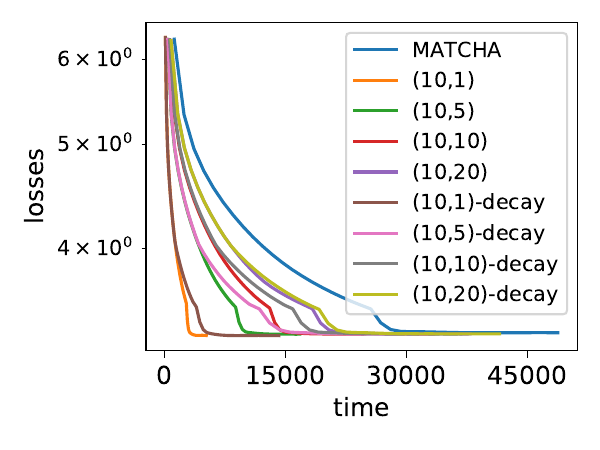}}\\
	\vspace{-0.1in} 
	\subfloat[Test acc on CIFAR 10]{
		\includegraphics[width=0.3\textwidth] {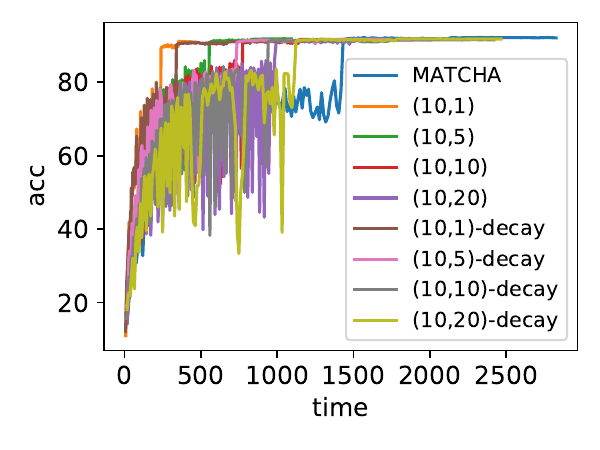}}
	\hspace{-0.15in} 
	\subfloat[Test acc on CIFAR 100]{
		\includegraphics[width=0.3\textwidth] {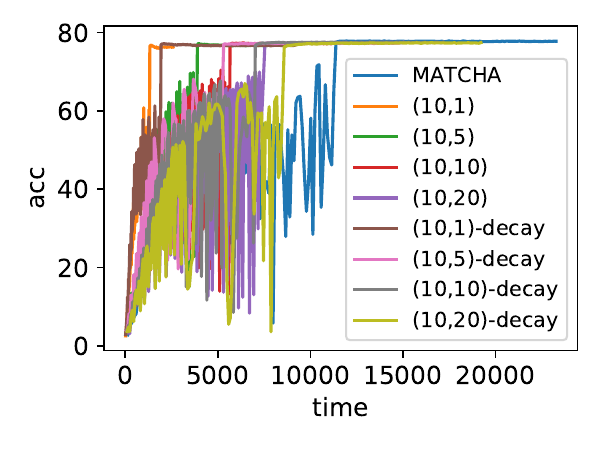}}
	\caption{LD-SGD and its variants with randomized $\W$ on three datasets. 
	The first row shows losses v.s. real time, while the second row shows test accuracy v.s. real time.
}
	\label{fig:miid}
	\vspace{-0.1in} 
\end{figure*}


\paragraph{Random $\W$}
In the body, we focus on the case where $\W$ is fixed and deterministic at each communication round.
MATCHA~\cite{wang2018adaptive} proposes to improve communication efficiency by carefully designing a random $\W$ for D-SGD.
In particular, MATCHA decomposes the base communication graph into
total several disjoint matchings and activate a small portion of the matchings at each communication round.
Mathematically speaking, MATCHA is identical to LD-SGD by letting $(I_1, I_2) = (0, 1)$ and randomizing $\W$.
It is natural to apply such randomized $\W$ to LD-SGD as a further extension.
In this case, $\W$ is independently generated and conforms to the distribution $\DM$ given in~\cite{wang2018adaptive}.

We then test the performance of LD-SGD with the i.i.d. generated  sequence of $\W$'s.
The randomized $\W$ constructed in~\cite{wang2018adaptive} will only activate a small portion of devices, therefore, the straggler effect, which means some quick devices have to wait for slower devices to response, can be alleviated.
So intuitively the communication time will be reduced.

We show the results  in Figure~\ref{fig:miid}.
Recall that MATCHA is actually identical to LD-SGD with $(I_1, I_2) = (0, 1)$ and $\W$ randomized.
From the first row of Figure~\ref{fig:miid}, LD-SGD and its variants have advantages on real time convergence than MATCHA even a small portion of device participates in the training.
The reason is similar as before: local updates is still much cheaper than communication, even though the communication is much more affordable than before by using a randomized $\W$.
Trading local computation for less communication is still a good strategy to improve communication efficiency.
Besides, though MATCHA often has a lower training loss, LD-SGD and its variants can obtain comparative test accuracy using much less real time (see the second row of Figure~\ref{fig:miid}).

We also test our decay strategy in this part.
In terms of loss, the decayed LD-SGD often has a smaller final loss than its non-decayed LD-SGD.
However, if in terms of real time, the non-decay LD-SGD converges a little bit faster than the decayed one; 
indeed, the decayed LD-SGD will gradually increase communication frequency and communication is more time-consuming.
Again, by trading more local updates for smaller communication, we can save a lot of communication measured in real time.

\section{Conclusion}
In this paper we have proposed a meta algorithm LD-SGD, which can be specified with any update scheme $\IM_T$.
Our analytical framework has shown that with any $\IM_T$ that has sufficiently small $\text{gap}(\IM_T)$, LD-SGD converges under the setting of stochastic non-convex optimization and non-identically distributed data.
We have also presented two novel update schemes, one adding multiple D-SGDs (denoted $\IM_T^1$) and the other empirically reducing the length of local updates  (denoted by $\IM_T^2$) .
Both the schemes help improve communication efficiency theoretically and empirically.
The framework we proposed might help researchers to design more efficient update schemes.

\ifCLASSOPTIONcaptionsoff
\newpage
\fi



%
%
%
\bibliographystyle{IEEEtran}
\bibliography{bib/decentralized,bib/distributed,bib/ml,bib/optimization,bib/compression}

\begin{thebibliography}{10}
\providecommand{\url}[1]{#1}
\csname url@samestyle\endcsname
\providecommand{\newblock}{\relax}
\providecommand{\bibinfo}[2]{#2}
\providecommand{\BIBentrySTDinterwordspacing}{\spaceskip=0pt\relax}
\providecommand{\BIBentryALTinterwordstretchfactor}{4}
\providecommand{\BIBentryALTinterwordspacing}{\spaceskip=\fontdimen2\font plus
\BIBentryALTinterwordstretchfactor\fontdimen3\font minus
  \fontdimen4\font\relax}
\providecommand{\BIBforeignlanguage}[2]{{%
\expandafter\ifx\csname l@#1\endcsname\relax
\typeout{** WARNING: IEEEtran.bst: No hyphenation pattern has been}%
\typeout{** loaded for the language `#1'. Using the pattern for}%
\typeout{** the default language instead.}%
\else
\language=\csname l@#1\endcsname
\fi
#2}}
\providecommand{\BIBdecl}{\relax}
\BIBdecl

\bibitem{li2014efficient}
M.~Li, T.~Zhang, Y.~Chen, and A.~J. Smola, ``Efficient mini-batch training for
  stochastic optimization,'' in \emph{Proceedings of the 20th ACM SIGKDD
  international conference on Knowledge discovery and data mining}, 2014, pp.
  661--670.

\bibitem{li2014scaling}
M.~Li, D.~G. Andersen, J.~W. Park, A.~J. Smola, A.~Ahmed, V.~Josifovski,
  J.~Long, E.~J. Shekita, and B.-Y. Su, ``Scaling distributed machine learning
  with the parameter server,'' in \emph{USENIX Symposium on Operating Systems
  Design and Implementation (OSDI)}, 2014.

\bibitem{konevcny2015federated}
J.~Kone{\v{c}}n{\`y}, B.~McMahan, and D.~Ramage, ``Federated optimization:
  distributed optimization beyond the datacenter,'' \emph{arXiv preprint
  arXiv:1511.03575}, 2015.

\bibitem{mcmahan2017communication}
B.~McMahan, E.~Moore, D.~Ramage, S.~Hampson, and B.~A. y~Arcas,
  ``Communication-efficient learning of deep networks from decentralized
  data,'' in \emph{Artificial Intelligence and Statistics (AISTATS)}, 2017.

\bibitem{konevcny2017stochastic}
J.~Kone{\v{c}}n{\`y}, ``Stochastic, distributed and federated optimization for
  machine learning,'' \emph{arXiv preprint arXiv:1707.01155}, 2017.

\bibitem{li2019federated}
T.~Li, A.~K. Sahu, A.~Talwalkar, and V.~Smith, ``Federated learning:
  Challenges, methods, and future directions,'' \emph{arXiv preprint
  arXiv:1908.07873}, 2019.

\bibitem{sahu2019federated}
A.~K. Sahu, T.~Li, M.~Sanjabi, M.~Zaheer, A.~Talwalkar, and V.~Smith,
  ``Federated optimization for heterogeneous networks,'' \emph{arXiv preprint
  arXiv:1812.06127}, 2019.

\bibitem{lian2017can}
X.~Lian, C.~Zhang, H.~Zhang, C.-J. Hsieh, W.~Zhang, and J.~Liu, ``Can
  decentralized algorithms outperform centralized algorithms? a case study for
  decentralized parallel stochastic gradient descent,'' in \emph{Advances in
  Neural Information Processing Systems (NIPS)}, 2017.

\bibitem{he2016deep}
K.~He, X.~Zhang, S.~Ren, and J.~Sun, ``Deep residual learning for image
  recognition,'' in \emph{IEEE Conference on Computer Vision and Pattern
  Recognition (CVPR)}, 2016.

\bibitem{lin2018don}
T.~Lin, S.~U. Stich, and M.~Jaggi, ``Don't use large mini-batches, use local
  sgd,'' \emph{arXiv preprint arXiv:1808.07217}, 2018.

\bibitem{stich2018local}
S.~U. Stich, ``Local {SGD} converges fast and communicates little,''
  \emph{arXiv preprint arXiv:1805.09767}, 2018.

\bibitem{wang2018cooperative}
J.~Wang and G.~Joshi, ``Cooperative {SGD}: A unified framework for the design
  and analysis of communication-efficient {SGD} algorithms,'' \emph{arXiv
  preprint arXiv:1808.07576}, 2018.

\bibitem{wang2018adaptive}
------, ``Adaptive communication strategies to achieve the best error-runtime
  trade-off in local-update sgd,'' \emph{arXiv preprint arXiv:1810.08313},
  2018.

\bibitem{wang2019matcha}
J.~Wang, A.~K. Sahu, Z.~Yang, G.~Joshi, and S.~Kar, ``Matcha: Speeding up
  decentralized sgd via matching decomposition sampling,'' \emph{arXiv preprint
  arXiv:1905.09435}, 2019.

\bibitem{ram2010asynchronous}
S.~S. Ram, A.~Nedi{\'c}, and V.~V. Veeravalli, ``Asynchronous gossip algorithm
  for stochastic optimization: Constant stepsize analysis,'' in \emph{Recent
  Advances in Optimization and its Applications in Engineering}.\hskip 1em plus
  0.5em minus 0.4em\relax Springer, 2010, pp. 51--60.

\bibitem{yuan2016convergence}
K.~Yuan, Q.~Ling, and W.~Yin, ``On the convergence of decentralized gradient
  descent,'' \emph{SIAM Journal on Optimization}, vol.~26, no.~3, pp.
  1835--1854, 2016.

\bibitem{sirb2016consensus}
B.~Sirb and X.~Ye, ``Consensus optimization with delayed and stochastic
  gradients on decentralized networks,'' in \emph{2016 IEEE International
  Conference on Big Data (Big Data)}.\hskip 1em plus 0.5em minus 0.4em\relax
  IEEE, 2016, pp. 76--85.

\bibitem{lan2017communication}
G.~Lan, S.~Lee, and Y.~Zhou, ``Communication-efficient algorithms for
  decentralized and stochastic optimization,'' \emph{Mathematical Programming},
  pp. 1--48, 2017.

\bibitem{tang2018d}
H.~Tang, X.~Lian, M.~Yan, C.~Zhang, and J.~Liu, ``D2: Decentralized training
  over decentralized data,'' \emph{arXiv preprint arXiv:1803.07068}, 2018.

\bibitem{koloskova2019decentralized}
A.~Koloskova, S.~U. Stich, and M.~Jaggi, ``Decentralized stochastic
  optimization and gossip algorithms with compressed communication,''
  \emph{arXiv preprint arXiv:1902.00340}, 2019.

\bibitem{luo2019heterogeneity}
Q.~Luo, J.~He, Y.~Zhuo, and X.~Qian, ``Heterogeneity-aware asynchronous
  decentralized training,'' \emph{arXiv preprint arXiv:1909.08029}, 2019.

\bibitem{seide2014bit}
F.~Seide, H.~Fu, J.~Droppo, G.~Li, and D.~Yu, ``1-bit stochastic gradient
  descent and its application to data-parallel distributed training of speech
  {DNNs},'' in \emph{Fifteenth Annual Conference of the International Speech
  Communication Association}, 2014.

\bibitem{lin2017deep}
Y.~Lin, S.~Han, H.~Mao, Y.~Wang, and W.~J. Dally, ``Deep gradient compression:
  Reducing the communication bandwidth for distributed training,'' \emph{arXiv
  preprint arXiv:1712.01887}, 2017.

\bibitem{zhang2017zipml}
H.~Zhang, J.~Li, K.~Kara, D.~Alistarh, J.~Liu, and C.~Zhang, ``Zipml: Training
  linear models with end-to-end low precision, and a little bit of deep
  learning,'' in \emph{Proceedings of the 34th International Conference on
  Machine Learning-Volume 70}.\hskip 1em plus 0.5em minus 0.4em\relax JMLR.
  org, 2017, pp. 4035--4043.

\bibitem{tang2018communication}
H.~Tang, S.~Gan, C.~Zhang, T.~Zhang, and J.~Liu, ``Communication compression
  for decentralized training,'' in \emph{Advances in Neural Information
  Processing Systems}, 2018, pp. 7652--7662.

\bibitem{wang2018atomo}
H.~Wang, S.~Sievert, S.~Liu, Z.~Charles, D.~Papailiopoulos, and S.~Wright,
  ``Atomo: Communication-efficient learning via atomic sparsification,'' in
  \emph{Advances in Neural Information Processing Systems}, 2018, pp.
  9850--9861.

\bibitem{horvath2019stochastic}
S.~Horv{\'a}th, D.~Kovalev, K.~Mishchenko, S.~Stich, and P.~Richt{\'a}rik,
  ``Stochastic distributed learning with gradient quantization and variance
  reduction,'' \emph{arXiv preprint arXiv:1904.05115}, 2019.

\bibitem{zhang2013communication}
Y.~Zhang, J.~C. Duchi, and M.~J. Wainwright, ``Communication-efficient
  algorithms for statistical optimization,'' \emph{Journal of Machine Learning
  Research}, vol.~14, pp. 3321--3363, 2013.

\bibitem{zhang2015divide}
Y.~Zhang, J.~Duchi, and M.~Wainwright, ``Divide and conquer kernel ridge
  regression: a distributed algorithm with minimax optimal rates,''
  \emph{Journal of Machine Learning Research}, vol.~16, pp. 3299--3340, 2015.

\bibitem{lee2017communication}
J.~D. Lee, Q.~Liu, Y.~Sun, and J.~E. Taylor, ``Communication-efficient sparse
  regression,'' \emph{The Journal of Machine Learning Research}, vol.~18,
  no.~1, pp. 115--144, 2017.

\bibitem{lin2017distributed}
S.-B. Lin, X.~Guo, and D.-X. Zhou, ``Distributed learning with regularized
  least squares,'' \emph{Journal of Machine Learning Research}, vol.~18, no.~1,
  pp. 3202--3232, 2017.

\bibitem{wang2019sharper}
S.~Wang, ``A sharper generalization bound for divide-and-conquer ridge
  regression,'' in \emph{The Thirty-Third AAAI Conference on Artificial
  Intelligence (AAAI)}, 2019.

\bibitem{smith2016cocoa}
V.~Smith, S.~Forte, C.~Ma, M.~Takac, M.~I. Jordan, and M.~Jaggi, ``{CoCoA}: A
  general framework for communication-efficient distributed optimization,''
  \emph{arXiv preprint arXiv:1611.02189}, 2016.

\bibitem{smith2017federated}
V.~Smith, C.-K. Chiang, M.~Sanjabi, and A.~S. Talwalkar, ``Federated multi-task
  learning,'' in \emph{Advances in Neural Information Processing Systems
  (NIPS)}, 2017.

\bibitem{hong2018gradient}
M.~Hong, M.~Razaviyayn, and J.~Lee, ``Gradient primal-dual algorithm converges
  to second-order stationary solution for nonconvex distributed optimization
  over networks,'' in \emph{International Conference on Machine Learning
  (ICML)}, 2018.

\bibitem{shamir2014communication}
O.~Shamir, N.~Srebro, and T.~Zhang, ``Communication-efficient distributed
  optimization using an approximate {N}ewton-type method,'' in
  \emph{International conference on machine learning (ICML)}, 2014.

\bibitem{zhang2015disco}
Y.~Zhang and X.~Lin, ``{DiSCO}: distributed optimization for self-concordant
  empirical loss,'' in \emph{International Conference on Machine Learning
  (ICML)}, 2015.

\bibitem{reddi2016aide}
S.~J. Reddi, J.~Kone{{c}}n{\`y}, P.~Richt{\'a}rik, B.~P{\'o}cz{\'o}s, and
  A.~Smola, ``{AIDE}: {f}ast and communication efficient distributed
  optimization,'' \emph{arXiv preprint arXiv:1608.06879}, 2016.

\bibitem{wang2018giant}
{S}husen {W}ang, {F}arbod~{R}oosta {K}horasani, {P}eng {X}u, and {M}ichael
  {W}.~{M}ahoney, ``{GIANT}: Globally improved approximate newton method for
  distributed optimization,'' in \emph{{C}onference on {N}eural {I}nformation
  {P}rocessing {S}ystems ({{NeurIPS}})}, 2018.

\bibitem{mahajan2018efficient}
D.~Mahajan, N.~Agrawal, S.~S. Keerthi, S.~Sellamanickam, and L.~Bottou, ``An
  efficient distributed learning algorithm based on effective local functional
  approximations,'' \emph{Journal of Machine Learning Research}, vol.~19,
  no.~1, pp. 2942--2978, 2018.

\bibitem{zinkevich2010parallelized}
M.~Zinkevich, M.~Weimer, L.~Li, and A.~J. Smola, ``Parallelized stochastic
  gradient descent,'' in \emph{Advances in neural information processing
  systems}, 2010, pp. 2595--2603.

\bibitem{you2018imagenet}
Y.~You, Z.~Zhang, C.-J. Hsieh, J.~Demmel, and K.~Keutzer, ``Imagenet training
  in minutes,'' in \emph{Proceedings of the 47th International Conference on
  Parallel Processing}.\hskip 1em plus 0.5em minus 0.4em\relax ACM, 2018, p.~1.

\bibitem{assran2019stochastic}
M.~Assran, N.~Loizou, N.~Ballas, and M.~Rabbat, ``Stochastic gradient push for
  distributed deep learning,'' in \emph{International Conference on Machine
  Learning}, 2019, pp. 344--353.

\bibitem{wang2019slowmo}
J.~Wang, V.~Tantia, N.~Ballas, and M.~Rabbat, ``Slowmo: Improving
  communication-efficient distributed sgd with slow momentum,'' \emph{arXiv
  preprint arXiv:1910.00643}, 2019.

\bibitem{konevcny2016federated}
J.~Kone{{c}}n{\`y}, H.~B. McMahan, D.~Ramage, and P.~Richt{\'a}rik, ``Federated
  optimization: distributed machine learning for on-device intelligence,''
  \emph{arXiv preprint arXiv:1610.02527}, 2016.

\bibitem{li2019convergence}
X.~Li, K.~Huang, W.~Yang, S.~Wang, and Z.~Zhang, ``On the convergence of
  {FedAvg} on {Non-IID} data,'' \emph{arXiv:1907.02189}, 2019.

\bibitem{haddadpour2019convergence}
F.~Haddadpour and M.~Mahdavi, ``On the convergence of local descent methods in
  federated learning,'' \emph{arXiv preprint arXiv:1910.14425}, 2019.

\bibitem{bianchi2013performance}
P.~Bianchi, G.~Fort, and W.~Hachem, ``Performance of a distributed stochastic
  approximation algorithm,'' \emph{IEEE Transactions on Information Theory},
  vol.~59, no.~11, pp. 7405--7418, 2013.

\bibitem{zhou2017convergence}
F.~Zhou and G.~Cong, ``On the convergence properties of a k-step averaging
  stochastic gradient descent algorithm for nonconvex optimization,''
  \emph{arXiv preprint arXiv:1708.01012}, 2017.

\bibitem{yu2019parallel}
H.~Yu, S.~Yang, and S.~Zhu, ``Parallel restarted sgd with faster convergence
  and less communication: Demystifying why model averaging works for deep
  learning,'' in \emph{AAAI Conference on Artificial Intelligence}, 2019.

\bibitem{bottou2018optimization}
L.~Bottou, F.~E. Curtis, and J.~Nocedal, ``Optimization methods for large-scale
  machine learning,'' \emph{Siam Review}, vol.~60, no.~2, pp. 223--311, 2018.

\bibitem{haddadpour19a}
F.~Haddadpour, M.~M. Kamani, M.~Mahdavi, and V.~Cadambe, ``Trading redundancy
  for communication: Speeding up distributed {SGD} for non-convex
  optimization,'' in \emph{International Conference on Machine Learning
  (ICML)}, 2019.

\bibitem{yu2019linear}
H.~Yu, R.~Jin, and S.~Yang, ``On the linear speedup analysis of communication
  efficient momentum {SGD} for distributed non-convex optimization,'' in
  \emph{International Conference on Machine Learning (ICML)}, 2019.

\bibitem{press2016using}
O.~Press and L.~Wolf, ``Using the output embedding to improve language
  models,'' \emph{arXiv preprint arXiv:1608.05859}, 2016.

\bibitem{jiang2017collaborative}
Z.~Jiang, A.~Balu, C.~Hegde, and S.~Sarkar, ``Collaborative deep learning in
  fixed topology networks,'' in \emph{Advances in Neural Information Processing
  Systems}, 2017, pp. 5904--5914.

\end{thebibliography}

\newpage
\appendices
\section{Proof of Main Result}
\label{append:PD-SGD}

\subsection{Additional notation}
\label{appen:add-nota-1}
In the proofs we will use the following notation.
Let $\G( \X; \xi)$ be defined in Section~\ref{sec:notation} previously. Let
\begin{align*}
\overline{\g}( \X; \xi) :=\frac{1}{n}\G(\X; \xi) \, \1_n= \frac{1}{n} \sum_{k=1}^n F_k(\x^{(k)};\xi^{(k)}) \in \RB^{d }
\end{align*}
be the averaged gradient.
Recall from~\eqref{eq:goal} the definition $f_k(\x) := \EB_{\xi \sim \DM_k} \big[ F_{k} \left( \x ; \xi\right) \big]$.
We analogously define
\begin{align*}
\nabla f(\X)
& := \EB \big[ \G( \X; \xi) \big]= \left[ \nabla f_1(\x^{(1)}), \cdots, \nabla f_n(\x^{(n)})\right] \in \RB^{d \times n},\\
\overline{\nabla f}(\X) 
&:= \EB  \big[ \overline{\g}( \X; \xi) \big] = \frac{1}{n}\nabla f(\X) \1_n = \frac{1}{n} \sum_{k=1}^n \nabla f_k(\x^{(k)})  \in \RB^{d },\\
\nabla f( \overline{\x} ) &:= \overline{\nabla f}( \overline{\x} )  =  \frac{1}{n} \sum_{k=1}^n \nabla f_k(\overline{\x})  \in \RB^{d } .
\end{align*}
Let $\Q = \frac{1}{n}\1_n \1_n^\top $ and $\overline{\x}_t = \frac{1}{n} \sum_{k=1}^n \x_t^{(k)}$.
Define the residual error as
\begin{equation}
\label{eq:v}
V_t = \EB_{\xi} \frac{1}{n} \big\|\X_t(\I-\Q) \big\|_F^2 =\EB_{\xi} \frac{1}{n} \sum_{k=1}^n \big\| \x_t^{(k)} - \overline{\x}_t \big\|^2 .
\end{equation}
where the expectation is taken with respect to all randomness of stochastic gradients or equivalently $\xi = (\xi_1, \cdots, \xi_t, \cdots)$ where $\xi_s = (\xi_s^{(1)}, \cdots, \xi_s^{(n)})^\top \in \RB^n$. Except where noted, we will use notation $\EB(\cdot)$ in stead of $\EB_{\xi}(\cdot)$ for simplicity. Hence $V_t = \frac{1}{n} \EB \sum_{k=1}^n \big\| \x_t^{(k)} - \overline{\x}_t \big\|^2$.

As mentioned in Section~\ref{sec:notation}, LD-SGD with arbitrary update scheme can be equivalently written in matrix form which will be used repeatedly in the following convergence analysis. Specifically, 
\begin{equation}
\label{eq:pdsgd}
\X_{t+1} = (\X_{t} - \G(\X_t; \xi_t))\W_t
\end{equation}
where $\X_t \in \RB^{d \times n}$ is the concatenation of $\{\x_t^{(k)} \}_{k=1}^n$, $\G(\X_t;\xi_t) \in \RB^{d \times n}$ is the concatenated gradient evaluated at $\X_t$ with the sampled datum $\xi_t$, and $\W_t \in \RB^{n \times n}$ is the connected matrix defined by
\begin{equation}
\label{eq:W}
\W_t = \left\{ \begin{array}{ll}
\I_n & \text{if} \ t  \notin \IM_T ; \\
\W & \text{if} \ t  \in \IM_T.
\end{array}\right. 
\end{equation}

\subsection{Useful lemmas}

The main idea of proof is to express $\X_t$ in terms of gradients and then develop upper bound on residual errors.
The technique to bound residual errors can be found in~\cite{wang2018adaptive,wang2018cooperative,wang2019matcha,yu2019parallel,yu2019linear}.

\begin{lem}[One step recursion]
	\label{lem:one-step}
	Let Assumptions~\ref{asum:smooth} and~\ref{asum:within-var} hold and $L$ and $\sigma$ be defined therein.
	Let $\eta$ be the learning rate.
	Then the iterate obtained from the update rule~\eqref{eq:pdsgd} satisfies
	\begin{align}
	\label{eq:one-step}
	\EB &\big[ f(\overline{\x}_{t+1}) \big]
	\: \leq \: \EB \big[ f(\overline{\x}_t) \big] - \frac{\eta}{2}(1-\eta L) \EB \big\| \overline{\nabla f}(\X_t) \big\|^2  \nonumber \\
	&\quad - \frac{\eta}{2} \EB \big\| \nabla f (\overline{\x}_t) \big\|^2 + \frac{L\sigma^2\eta^2 }{2n} + \frac{\eta L^2}{2} V_t,
	\end{align}
	where the expectations  are taken with respect to all randomness in stochastic gradients.
\end{lem}

\begin{proof}
	Recall that from the update rule~\eqref{eq:pdsgd} we have
	\[ \overline{\x}_{t+1} = \overline{\x}_{t} - \eta \overline{\g}(\X_t, \xi_t). \]
	When Assumptions~\ref{asum:smooth} and~\ref{asum:within-var} hold,
	it follows directly from Lemma 8 in~\cite{tang2018d} that
	\begin{align*}
	\EB \big[ f(\overline{\x}_{t+1}) \big]
	&\: \leq \: \EB \big[ f(\overline{\x}_t) \big] - \frac{\eta}{2} \EB \big\| \nabla f (\overline{\x}_t) \big\|^2  \\
&\quad - \frac{\eta}{2}(1-\eta L) \EB \big\| \overline{\nabla f}(\X_t) \big\|^2 
	+ \frac{L\sigma^2\eta^2 }{2n} \\
	&\quad + \frac{\eta}{2}  \EB \big\| \nabla f(\overline{\x}_t) - \overline{\nabla f}(\X_t) \big\|^2.
	\end{align*}
	
	The conclusion then follows from 
	\begin{align*}
	\EB \| \nabla f(\overline{\x}_t) - \overline{\nabla f}(\X_t)  \|^2
	&= \frac{1}{n^2} \EB \bigg\| \sum_{k=1}^n \left[  f_k \big(\overline{\x}_t \big) -   f_k \big(\x_t^{(k)} \big) \right] \bigg\|^2\\
	& \overset{(a)}{\le} \frac{1}{n} \sum_{k=1}^n \EB \Big\|  f_k \big(\overline{\x}_t \big) -   f_k \big(\x_t^{(k)} \big)  \Big\|^2\\
	& \overset{(b)}{\le}  \frac{L^2}{n} \sum_{k=1}^n  \EB \big\| \x_t^{(k)} - \overline{\x}_t \big\|^2\\
	& = L^2 V_t
	\end{align*}
	where (a) follows from Jensen's inequality, (b) follows from Assumption~\ref{asum:smooth},
	and $V_t$ is defined in \eqref{eq:v}.
\end{proof}

\begin{lem}[Residual error decomposition]
	\label{lem:residual-decom}
	Let $\X_1 = \x_1 \1_n^\top \in \RB^{d\times n}$ be the initialization. If we apply the update rule~\eqref{eq:pdsgd}, then for any $t \ge 2$,
	\begin{equation}
	\label{eq:decomposition}
	\X_t(\I_n-\Q)= - \eta \sum_{s=1}^{t-1} \G(\X_s;\xi_s) \left(  \Ph_{s, t-1} - \Q \right)
	\end{equation}
	where $\Ph_{s, t-1}$ is defined in~\eqref{eq:Ph} and $\W_t$ is given in~\eqref{eq:W}.
	\begin{equation}
	\label{eq:Ph}
	\Ph_{s, t-1} = \left\{ 
	\begin{array}{ll}
	\I_{n}   &\text{if} \ s \ge t \\
	\prod_{l=s}^{t-1} \W_l & \text{if} \ s < t.
	\end{array}
	\right.
	\end{equation}
\end{lem}

\begin{proof}[Proof]
	For convenience, we denote by $\G_t = \G(\X_t; \xi_t) \in \RB^{d \times n}$ the concatenation of stochastic gradients at iteration $t$. According to the update rule, we have
	\begin{align*}
	\X_t&(\I_n-\Q) = (\X_{t-1} - \eta\G_{t-1})\W_{t-1}(\I_n-\Q)\\
	&\overset{(a)}{=} \X_{t-1}(\I_n-\Q)\W_{t-1} - \eta\G_{t-1}(\W_{t-1} -\Q)\\
	&\overset{(b)}{=} \X_{t-l}(\I_n-\Q) \prod_{s=t-l}^{t-1} \W_{s} - \eta \sum_{s=t-l}^{t-1}  \G_s (\Ph_{s, t-1} - \Q ) \\
	&\overset{(c)}{=} \X_1(\I_n-\Q)\Ph_{1, t-1} - \eta \sum_{s=1}^{t-1} \G_s \left(  \Ph_{s, t-1} - \Q \right)
	\end{align*}
	where (a) follows from $\W_{t-1}\Q = \Q \W_{t-1}$; (b) results by iteratively expanding the expression of $\X_{s}$ from $s = t-1$ to $s = t- l+1$ and plugging in the definition of $\Ph_{s, t-1}$ in~\eqref{eq:W}; (c) follows simply by setting $l=t-1$. Finally, the conclusion follows from the initialization $\X_1 = \x_1 \1_n^\top$ which implies $\X_1(\I-\Q)= \0$.
\end{proof}

\begin{lem}[Gradient variance decomposition]
	\label{lem:grad-decom}
	Given any sequence of deterministic matrices $\{\A_s\}_{s =1}^t$, then for any $t \ge 1$,
	\begin{align}
	\label{eq:grad-var}
	&\EB_{\xi} \bigg\| \sum_{s =1}^t  \left[ \G \big(\X_s;\xi_s \big) - \nabla f \big(\X_s \big) \right] \A_s \bigg\|_F^2  \nonumber \\
	&=  	\sum_{s =1}^t  \EB_{\xi_s}  \Big\|  \left[ \G \big(\X_s;\xi_s \big) - \nabla f \big(\X_s \big) \right] \A_s \Big\|_F^2 .
	\end{align}
	where the expectation $\EB_{\xi}(\cdot)$ is taken with respect to the randomness of $\xi = (\xi_1, \cdots, \xi_t, \cdots)$ and $\EB_{\xi_s}(\cdot)$ is with respect to $\xi_s = (\xi_s^{(1)}, \cdots, \xi_s^{(n)})^\top \in \RB^n$.
\end{lem}

\begin{proof}
	\begin{align*}
	\EB_{\xi} &\bigg\| \sum_{s =1}^t  \left[ \G \big(\X_s;\xi_s \big) - \nabla f \big(\X_s \big) \right] \A_s \bigg\|_F^2\\
	&\: = \:  \sum_{s =1}^t  \EB_{\xi_s} \Big\| \left[  \G \big(\X_s;\xi_s \big) - \nabla f \big(\X_s \big)  \right] \A_s \Big\|_F^2 \\
	& + 2 \sum_{1 \le s < l \le t}  \EB_{\xi_s, \xi_l} \Big[  \Big\langle \left[ \G \big(\X_s;\xi_s \big) - \nabla f \big(\X_s \big) \right] \A_s, \\
	&\qquad  \qquad \qquad \qquad  \: \left[ \G \big(\X_l;\xi_l \big) - \nabla f \big(\X_l \big) \right] \A_l  \Big\rangle \Big]
	\end{align*}
	where the inner product of matrix is defined by $\langle \A, \B \rangle  =\text{tr}(\A \B^\top)$.
	
	Since different nodes work independently without interference, for $ s \neq l \in [t]$, $\xi_s$ is independent with $\xi_l$. 
	Let $\FM_s = \sigma(\{\xi_l\}_{l=1}^s)$ be the $\sigma$-field generated by all the random variables until iteration $s$.  
	Then for any $1 \le s < l \le t$, we obtain
	\begin{align*}
	&\EB_{\xi_s, \xi_l} \Big[ \Big\langle  \left( \G(\X_s;\xi_s) - \nabla f(\X_s) \right) \A_s, \\
		&\qquad  \qquad \qquad  \left( \G(\X_l;\xi_l) - \nabla f(\X_l)\right) \A_l  \Big\rangle \Big] \\
	&=\EB_{\xi_s} \EB_{\xi_l} \Big[ \Big\langle  \left( \G(\X_s;\xi_s) - \nabla f(\X_s) \right) \A_s, \\
	&\qquad  \qquad \qquad  \left( \G(\X_l;\xi_l) - \nabla f(\X_l)\right) \A_l  \Big\rangle \, \Big| \, \FM_{l-1} \Big] \\
	&\overset{(a)}{=}\EB_{\xi_s}  \Big\{ \Big\langle  \left( \G(\X_s;\xi_s) - \nabla f(\X_s) \right) \A_s,   \\
	&\qquad  \qquad \qquad 
	\: \EB_{\xi_l} \big[ \left( \G(\X_l;\xi_l) - \nabla f(\X_l)\right) \A_l  \big| \FM_{l-1} \big] \Big\rangle \Big\} \\
	&\overset{(b)}{=}\EB_{\xi_s}  \Big[ \Big\langle  \G(\X_s;\xi_s) - \nabla f(\X_s), \: \0 \Big\rangle \Big] = 0
	\end{align*}
	where (a) follows from the tower rule by noting that $\X_s$ and $\xi_s$ are both $\FM_{l-1}$-measurable and (b) uses the fact that $\xi_l$ is independent with $\FM_s (s < l)$ and  $\G(\X_l;\xi_l)$ is a unbiased estimator of $\nabla f(\X_l)$.
\end{proof}

\begin{lem}[Bound on second moments of gradients]
	\label{lem:grad-second}
	For any $n$ points: $\{ \x_t^{(k)}\}_{k=1}^n$, define $\X_t = \left[ \x_t^{(1)}, \cdots, \x_t^{(n)} \right]$ as their concatenation, then under Assumption~\ref{asum:smooth} and~\ref{asum:inter-var},
	\begin{equation}
	\label{eq:grad-second}
	\frac{1}{n} \EB \big\| \nabla f(\X_t)  \big\|_F^2 
	\: \leq \: 8L^2 V_t + 4 \kappa^2  +  4 \EB \big\| \overline{\nabla f}(\X_t)  \big\|^2.
	\end{equation}
\end{lem}

\begin{proof}
	By splitting $\nabla f(\X_t)$ into four terms, we obtain
	\begin{align*}
	& \EB \| \nabla f(\X_t)  \|_F^2\\
	=&\EB  \| \nabla f(\X_t)  - \nabla f(\overline{\x}_t  \1_n^\top) 
	+  \nabla f(\overline{\x}_t  \1_n^\top)  -  \nabla f(\overline{\x}_t) \1_n^\top  \\
	& \quad +  \nabla f(\overline{\x}_t) \1_n^\top - \overline{\nabla f}(\X_t)\1_n^\top 
	+ \overline{\nabla f}(\X_t)\1_n^\top  \|_F^2 \\
	\overset{(a)}{\le} 
	& 4\EB \| \nabla f(\X_t)  - \nabla f(\overline{\x}_t  \1_n^\top)\|_F^2 \\
	& \quad + 4 \EB \| \nabla f(\overline{\x}_t  \1_n^\top)  -  \nabla f(\overline{\x}_t) \1_n^\top  \|_F^2    \\
	& \quad + 4\EB \| \nabla f(\overline{\x}_t) \1_n^\top - \overline{\nabla f}(\X_t)\1_n^\top \|_F^2 + 4 \EB \| \overline{\nabla f}(\X_t)\1_n^\top  \|_F^2 \\
	\overset{(b)}{=}& 4L^2 n V_t +  4 \EB \| \nabla f(\overline{\x}_t  \1_n^\top)  -  \nabla f(\overline{\x}_t) \1_n^\top \|_F^2  \\
	& \quad + 4L^2 n V_t + 4 n \EB  \| \overline{\nabla f}(\X_t)  \|^2\\
	\overset{(c)}{=}&  8L^2 n V_t  + 4 n \kappa^2 +  4 n \| \overline{\nabla f}(\X_t)  \|^2
	\end{align*}
	where (a) follows from the basic inequality $\| \sum_{i=1}^n \A_i \|_F^2 \le n \sum_{i=1}^n \|\A_i\|_F^2$; (b) follows from the smoothness of $\{f_k\}_{k=1}^n$ and $f = \frac{1}{n} \sum_{k=1}^n f_k$ (Assumption~\ref{asum:smooth}) and the definition of $V_t$ in~\eqref{eq:v}; (c) follows from Assumption~\ref{asum:inter-var} as a result of the fact $\|\nabla f(\overline{\x}_t  \1_n^\top)  -  \nabla f(\overline{\x}_t) \1_n^\top\|_F^2 = \sum_{k=1}^n \| \nabla f_k(\overline{\x}_t) - \nabla f(\overline{\x}_t) \|^2$.
\end{proof}

\begin{lem}[Bound on residual errors]
	\label{lem:residual-bound}
	Let $\rho_{s, t-1} = \|\Ph_{s, t-1} - \Q\|$ where $\Ph_{s, t-1}$ is defined in~\eqref{eq:Ph}. Then the residual error can be upper bounded, i.e., 
	\[  V_t \le 2 \eta^2 U_t  \]
	where
	\begin{equation}
	\label{eq:U}
	U_t =
	\sigma^2  \sum_{s=1}^{t-1} \rho^2_{s, t-1} 
	+
	\left( \sum_{s=1}^{t-1}  \rho_{s, t-1} \right) \left(  \sum_{s=1}^{t-1} \rho_{s, t-1}   L_s \right).
	\end{equation}
	and
	\[ L_s = 8L^2 V_s  + 4\kappa^2 +4 \EB \| \overline{\nabla f}(\X_s)  \|^2.   \]
\end{lem}

\begin{proof}
	Again we denote by $\G_t = \G(\X_t; \xi_t)$ for simplicity. From Lemma~\ref{lem:residual-decom}, we can obtain a closed form of $V_t$. Then it follows that
	\begin{align*}
	&n V_t  \\
	&= \EB \| \X_{t}(\I - \Q) \|_F^2 = \eta^2 \EB \bigg\| \sum_{s=1}^{t-1} \G_s(\Ph_{s, t-1} - \Q)\bigg\|_F^2 \\
	&=\eta^2 \EB \bigg\| \sum_{s=1}^{t-1} (\G_s - \EB \G_s )(\Ph_{s, t-1} - \Q)  +  \sum_{s=1}^{t-1} \EB \G_s(\Ph_{s, t-1} - \Q) \bigg\|_F^2 \\
	&\overset{(a)}{\le} 2 \eta^2 \EB\bigg\| \sum_{s=1}^{t-1} (\G_s - \nabla f (\X_s) )(\Ph_{s, t-1} - \Q)\bigg\|_F^2 \\
	& \quad + 2\eta^2 \EB \bigg\|\sum_{s=1}^{t-1} \nabla f (\X_s)(\Ph_{s, t-1} - \Q) \bigg\|_F^2\\
	&\overset{(b)}{=} 2 \eta^2\EB \sum_{s=1}^{t-1} \|(\G_s - \nabla f (\X_s) )(\Ph_{s, t-1} - \Q)\|_F^2 \\
	& \quad+ 2\eta^2 \EB \bigg\|\sum_{s=1}^{t-1} \nabla f (\X_s)(\Ph_{s, t-1} - \Q) \bigg\|_F^2\\
	&\overset{(c)}{\le} 2 \eta^2 \EB \sum_{s=1}^{t-1} \|(\G_s - \nabla f (\X_s) )(\Ph_{s, t-1} - \Q)\|_F^2 \\
	& \quad+ 2\eta^2 \EB \left(  \sum_{s=1}^{t-1} \| \nabla f (\X_s)(\Ph_{s, t-1} - \Q) \|_F  \right)^2 \\
	&\overset{(d)}{\le} 2 \eta^2 \EB \sum_{s=1}^{t-1} \|\G_s - \nabla f (\X_s) \|_F^2 \|\Ph_{s, t-1} - \Q\|^2  \\
	& \quad+ 2\eta^2 \EB \left(  \sum_{s=1}^{t-1} \| \nabla f (\X_s)\|_F \|(\Ph_{s, t-1} - \Q) \| \right)^2 \\
	&\overset{(e)}{=} 2 \eta^2 \sum_{s=1}^{t-1} \rho^2_{s, t-1}  \EB \|\G_s - \nabla f (\X_s) \|_F^2 \\
	& \quad  + 2\eta^2 \EB \left(  \sum_{s=1}^{t-1} \rho_{s, t-1} \| \nabla f (\X_s)\|_F  \right)^2 \\
	&\overset{(f)}{\le} 2 \eta^2 \sum_{s=1}^{t-1} \rho^2_{s, t-1} \EB  \|\G_s - \nabla f (\X_s) \|_F^2 \\
	& \quad  + 2\eta^2 \left( \sum_{s=1}^{t-1}  \rho_{s, t-1} \right) \left(  \sum_{s=1}^{t-1} \rho_{s, t-1}\EB  \|  \nabla f (\X_s)\|_F^2 \right) \\
	&\overset{(g)}{\le} 2 \eta^2 \sum_{s=1}^{t-1} \rho^2_{s, t-1} n \sigma^2   + 2\eta^2 \left( \sum_{s=1}^{t-1}  \rho_{s, t-1} \right) \left(  \sum_{s=1}^{t-1} \rho_{s, t-1}\cdot nL_s \right) \\
	&= 2n\eta^2 \left[   \sigma^2  \sum_{s=1}^{t-1} \rho^2_{s, t-1} 
	+  \left( \sum_{s=1}^{t-1}  \rho_{s, t-1} \right) \left(  \sum_{s=1}^{t-1} \rho_{s, t-1}  L_s \right)  \right] \\
	&=2n\eta^2 U_t
	\end{align*}
	where (a) follows from the basic inequality $\| \a + \bb \|^2 \le 2 (\|\a\|^2 + \|\bb\|^2)$ and $\EB \G_s = \nabla f(\X_s)$; (b) follows from Lemma~\ref{lem:grad-decom}; (c) follows from the triangle inequality $\|\sum_{s=1}^{t-1}\A_s\|_F \le \sum_{s=1}^{t-1}\|\A_s\|_F$; (d) follows from the basic inequality $\|\A\B\|_F \le \|\A\|_F \|\B\|$ for any matrix $\A$ and $\B$; (e) directly follows from the notation $\rho_{s, t-1} = \|\Ph_{s, t-1} - \Q\|$; (f) follows from the Cauchy inequality; (g) follows from Assumption~\ref{asum:within-var} and Lemma~\ref{lem:grad-second}.
\end{proof}

\begin{lem}[Bound on average residual error]
	\label{lem:residual-average}
	For any fix $T$, define
	\begin{gather}
		\label{eq:ABC}
	A_T =\frac{1}{T}\sum_{t=1}^{T} \sum_{s=1}^{t-1} \rho^2_{s, t-1}, \\
	B_T = \frac{1}{T} \sum_{t=1}^{T} \left(\sum_{s=1}^{t-1} \rho_{s, t-1} \right)^2, \\
	C_T = \max_{s \in [T-1]} \sum_{t=s+1}^T  \rho_{s, t-1}  \left( \sum_{l=1}^{t-1}  \rho_{l, t-1} \right).
	\end{gather}
	Assume the learning rate is so small that $16\eta^2L^2 C_T < 1$, then
	\begin{align*}
	\frac{1}{T}\sum_{t=1}^T V_t 
	& \le  \bigg[ A_T \sigma^2 + B_T \kappa^2   + C_T\frac{1}{T}\sum_{t=1}^{T}  \EB \| \overline{\nabla f}(\X_t)  \|^2  \bigg] \\
	& \cdot \frac{8\eta^2}{1-16\eta^2L^2 C_T}
	\end{align*}
\end{lem}

\begin{proof}
	Denote by $Z_s = 8L^2 V_s  + 4 \EB \| \overline{\nabla f}(\X_s)  \|^2 $ for short. Then $L_s = Z_s + 4\kappa^2$. 
	From Lemma~\ref{lem:residual-bound}, $V_t \le 2 \eta^2 U_t$, then 
	\begin{equation}
	\label{eq:V<U}
	\frac{1}{T}\sum_{t=1}^T V_t \le 2 \eta^2 \cdot \frac{1}{T} \sum_{t=1}^{T}U_t
	\end{equation}
	and
	\begin{align}
	\label{eq:U_average_u}
	&\frac{1}{T}\sum_{t=1}^{T}U_t \nonumber  \\
	&\overset{\eqref{eq:U}}{=} 
	\frac{1}{T}\sum_{t=1}^{T} \left[ 
	\sigma^2  \sum_{s=1}^{t-1} \rho^2_{s, t-1} 
	+  \left( \sum_{s=1}^{t-1}  \rho_{s, t-1} \right) \left(  \sum_{s=1}^{t-1} \rho_{s, t-1} L_s\right) \right] \nonumber \\
	&\overset{(a)}{=} \sigma^2 \frac{1}{T}\sum_{t=1}^{T} \sum_{s=1}^{t-1} \rho^2_{s, t-1}  + 4 \kappa^2 \frac{1}{T} \sum_{t=1}^{T} \left(\sum_{s=1}^{t-1} \rho_{s, t-1} \right)^2  \nonumber \\
	& \quad \quad + \frac{1}{T} \sum_{t=1}^T \left( \sum_{l=1}^{t-1}  \rho_{l, t-1} \right) \left( \sum_{s=1}^{t-1} \rho_{s, t-1} Z_s \right)  \nonumber \\
	&\overset{(b)}{=} \sigma^2 \frac{1}{T}\sum_{t=1}^{T} \sum_{s=1}^{t-1} \rho^2_{s, t-1} 
	+ 4 \kappa^2 \frac{1}{T} \sum_{t=1}^{T} \left(\sum_{s=1}^{t-1} \rho_{s, t-1} \right)^2 \nonumber \\
	& \quad \quad  + \frac{1}{T}\sum_{s=1}^{T-1} Z_s   \sum_{t=s+1}^T  \rho_{s, t-1}  \left( \sum_{l=1}^{t-1}  \rho_{l, t-1} \right)  \nonumber \\	  
	&\overset{(c)}{\le} \sigma^2 A_T + 4 \kappa^2  B_T
	+  C_T \frac{1}{T} \sum_{s=1}^{T-1} Z_s   \nonumber \\	 
	&\overset{(d)}{\le}\sigma^2 A_T + 4 \kappa^2  B_T
	+ 8 L^2 C_T  \frac{1}{T} \sum_{s=1}^{T-1} V_s  \nonumber \\
	& \quad \quad  + 4 C_T \frac{1}{T}\sum_{t=1}^{T}  \EB \| \overline{\nabla f}(\X_t)  \|^2 \nonumber  \\
	&\overset{(e)}{\le}\sigma^2 A_T + 4 \kappa^2  B_T
	+ 16 \eta^2 L^2 C_T  \frac{1}{T} \sum_{t=1}^{T} U_t \nonumber \\
	& \quad \quad + 4 C_T \frac{1}{T}\sum_{t=1}^{T}  \EB \| \overline{\nabla f}(\X_t)  \|^2   
	\end{align}
	where (a) follows from rearrangement;
	(b) follows from the equality that
	\begin{align*}
	&\sum_{t=1}^{T} \left( \sum_{s=1}^{t-1}  \rho_{s, t-1} \right) \left( \sum_{s=1}^{t-1} \rho_{s, t-1} Z_s \right)  \\
	&=  \sum_{t=1}^{T} \left( \sum_{l=1}^{t-1}  \rho_{l, t-1} \right) 
	\left( \sum_{s=1}^{T-1} \rho_{s, t-1} Z_s 1_{ \{t > s \}} \right) \\
	&=  \sum_{s=1}^{T-1} Z_s   \sum_{t=s+1}^T  \rho_{s, t-1}  \left( \sum_{l=1}^{t-1}  \rho_{l, t-1} \right); 
	\end{align*}
	(c) following from the notation~\eqref{eq:ABC};
	(d) follows from the definition of $\Z_s$;
	(e) follows from Lemma~\ref{lem:residual-bound}.

	By arranging~\eqref{eq:U_average_u} and assuming the learning rate is small enough such that $16\eta^2L^2 C_T < 1$, then we have
	\begin{align}
	\label{eq:U_average_u_2}
	\frac{1}{T}\sum_{t=1}^T U_t 
	&\le \left[ A_T \sigma^2 + 4B_T \kappa^2 +  4C_T\frac{1}{T}\sum_{t=1}^{T}  \EB \| \overline{\nabla f}(\X_t)  \|^2  \right] \nonumber \\
	& \quad \quad \cdot \frac{1}{1-16\eta^2L^2 C_T} \nonumber \\
	&\le \left[ A_T \sigma^2 + B_T \kappa^2 +  C_T\frac{1}{T}\sum_{t=1}^{T}  \EB \| \overline{\nabla f}(\X_t)  \|^2  \right] \nonumber \\
	& \quad \quad \cdot \frac{4}{1-16\eta^2L^2 C_T} 
	\end{align}	
	Our conclusion then follows by combining~\eqref{eq:V<U} and~\eqref{eq:U_average_u_2}.
\end{proof}

\begin{lem}[Computation of $\rho_{s, t-1}$]
	\label{lem:rho_compute}
	Define $\rho_{s, t-1} = 1$ for any $t \le s$ and $\rho_{s, t-1} =  \|\Ph_{s, t-1} - \Q\|$ when $s < t$. 
	Then $\rho_{s, t-1} = \prod_{l=s}^{t-1} \rho_l$ with $\rho_l = \rho$ if $l \in \IM_T$, else $\rho_l = 1$, where $\rho$ is defined in Assumption~\ref{asum:W}. 
	As a direct consequence, $\rho_{s, t-1} = \rho_{s, l-1} \rho_{l, t-1}$ for any $s\le l \le t$ and thus $\rho_{s, t-1} = \rho^{|[s:t-1] \cap \IM_T|}$.
\end{lem}
\begin{proof}
	By definition, we have $\rho_{s, t-1} = \|\Ph_{s, t-1} - \Q\| = \|\prod_{l=s}^{t-1}\W_l -\Q \|$. Since for any positive integer $l$, $\W_l \Q = \Q \W_l$, then $\W_l$ and $\Q$ can be simultaneously diagonalized. From this it is easy to see that $ \|\prod_{l=s}^{t-1}\W_l -\Q \|= \prod_{l=s}^{t-1} \rho_l$ where $\rho_l$ is the second largest absolute eigenvalue of $\W_l$. Note that $\W_l$ is either $\W$ or $\I$ according to the value of $l$ as a result of the definition~\eqref{eq:W}. 
	Hence  $\rho_l = \rho$ if $l \in \IM_T$, else $ = 1$.
\end{proof}

\subsection{Proof of Theorem~\ref{thm:PD-SGD}}
\label{append:main_proof}

\begin{proof}
	From Lemma~\ref{lem:one-step}, it follows that
	\begin{align*}
	\EB \big[ f(\overline{\x}_{t+1})  \big] 
	&\: \leq \: \EB \big[ f(\overline{\x}_t) \big] 
	- \frac{\eta}{2}(1-\eta L) \EB \big\| \overline{\nabla f}(\X_t) \big\|^2 \\
	&\quad- \frac{\eta}{2} \EB \big\| \nabla f (\overline{\x}_t) \big\|^2 
	+ \frac{L\sigma^2\eta^2 }{2n} 
	+ \frac{\eta L^2}{2} V_t. 
	\end{align*}
	Note that the expectation is taken with respect to all randomness of stochastic gradients, i.e., $\xi = (\xi_1, \xi_2. \cdots)$.
	Arranging this inequality, we have
	\begin{align}
	\label{eq:e1}
	\EB \big\| \nabla f (\overline{\x}_t) \big\|^2 
	&\: \leq \: \frac{2}{\eta} \Big\{ \EB \big[ f(\overline{\x}_t) \big] -  \EB \big[ f(\overline{\x}_{t+1}) \big] \Big\} + L^2 V_t\nonumber \\
	&- (1-\eta L) \EB \big\| \overline{\nabla f}(\X_t) \big\|^2 + \frac{L\sigma^2\eta}{n} .
	\end{align}
	
	Then it follows that
	\begin{align}
	\label{eq:e2}
	&\frac{1}{T} \sum_{t=1}^{T} \EB \| \nabla f (\overline{\x}_t)\|^2  \nonumber \\
	& \overset{(a)}{\le} \frac{2}{\eta T} \Big\{ \EB \big[ f(\overline{\x}_1) \big] -  \EB \big[ f(\overline{\x}_{T+1}) \big] \Big\}   + \frac{L\sigma^2\eta}{n}  \nonumber\\
	& \quad  \quad +  \frac{L^2}{T} \sum_{t=1}^{T}V_t 
	- \frac{1-\eta L }{T} \sum_{t=1}^{T}  \EB \big\| \overline{\nabla f}(\X_t) \big\|^2 \nonumber  \\
	& \overset{(b)}{\le}
	\frac{2}{\eta T} \Big\{ \EB \big[ f(\overline{\x}_1) \big] -  \EB \big[ f(\overline{\x}_{T+1}) \big] \Big\} 
	+ \frac{L\sigma^2\eta}{n}  \nonumber \\
	& \quad \quad  -  \frac{1-\eta L}{T} \sum_{t=1}^{T}  \EB \big\| \overline{\nabla f}(\X_t) \big\|^2 +
	\frac{8\eta^2L^2}{1-16\eta^2L^2 C_T} \times \nonumber \\
	& \quad  \quad \left[ A_T \sigma^2 + B_T \kappa^2 +  C_T\frac{1}{T}\sum_{t=1}^{T}  \EB \| \overline{\nabla f}(\X_t)  \|^2  \right]  \nonumber \\
	& \overset{(c)}{\le}
	\frac{2}{\eta T} \Big\{  \EB \big[ f(\overline{\x}_1) \big] -  \EB \big[ f(\overline{\x}_{T+1}) \big] \Big\}  \nonumber \\
	&  \quad \quad + \frac{L\sigma^2\eta}{n} + 16\eta^2L^2A_T \sigma^2 + 16\eta^2L^2B_T\kappa^2 \nonumber  \\
	& \quad \quad- (1-\eta L - 16\eta^2L^2C_T)  \frac{1}{T} \sum_{t=1}^{T}  \EB \| \overline{\nabla f}(\X_t)\|^2 \nonumber \\
	&\overset{(d)}{\le}
	\frac{2}{\eta T} \Big\{ \EB \big[ f(\overline{\x}_1) \big] -  \EB \big[ f(\overline{\x}_{T+1}) \big] \Big\}  + \frac{L\sigma^2\eta}{n}  \nonumber \\
	&\quad \quad  + 16\eta^2L^2A_T \sigma^2 + 16\eta^2L^2B_T\kappa^2  
	\end{align}
	where (a) follows by telescoping and averaging~\eqref{eq:e1}; 
	(b) follows from the upper bound of $\frac{1}{T} \sum_{t=1}^{T}V_t$ in Lemma~\ref{lem:residual-average};
	(c) follows from the choice of the learning rate $\eta$ which satisfies $ \frac{1}{1-16\eta^2L^2C_T} \le 2$ (since $16\eta^2L^2C_T \le \frac{1}{2}$ from~\eqref{eq:lr}) and rearrangement; 
	(d) follows the requirement that the learning rate $\eta$ is small enough such that $ \eta L + 16\eta^2L^2C_T < 1$ (which is satisfied since $\eta L \le \frac{1}{2}$ and $16\eta^2L^2K^2 \le \frac{1}{2}$).
\end{proof}

\subsection{Proof of Theorem~\ref{thm:bound}}
For any prescribed $\IM_T$, denote by $g = |\IM_T|$ and $\IM_T = \{ e_1, e_2, \cdots, e_g \}$ with $e_0 = 1 \le e_1 < e_2 < \cdots < e_g \le T=e_{g+1}$. 
For short, we let $s_i = e_i-e_{i-1}$ for $i \in [g+1]$.
Therefore, $\text{gap}(\IM_T) = \max_{i\in[g+1]} s_i$ from Definition~\ref{def:gap} and $T = \sum_{l=0}^{g} s_{l+1}$.

Recall the definition in~\eqref{eq:ABC}:
\begin{gather*}
A_T =\frac{1}{T}\sum_{t=1}^{T} \sum_{s=1}^{t-1} \rho^2_{s, t-1},  \\
B_T = \frac{1}{T} \sum_{t=1}^{T} \left(\sum_{s=1}^{t-1} \rho_{s, t-1} \right)^2,  \\
C_T = \max_{s \in [T-1]} \sum_{t=s+1}^T  \rho_{s, t-1}  \left( \sum_{l=1}^{t-1}  \rho_{l, t-1} \right).
\end{gather*}
In this section, we will provide proof for the bound on $A_T, B_T$ and $C_T$ in terms of $\text{gap}(\IM_T)$ (Theorem~\ref{thm:bound}).

\begin{proof}
	Recall that from Definition~\ref{def:rho} and Lemma~\ref{lem:rho_compute}, $\rho_{s, t-1} = \rho^{|[s: t-1] \cap \IM_T|}$, where $[s:t-1] = \{l \in \N: s \le l \le t-1\}$ and $\rho$ is defined in Assumption~\ref{asum:W}.
	For simplicity, let $\Delta = \text{gap}(\IM_T) = \max_{i\in[g]} s_i$.
	
	Let's first have a glance at $\sum_{s=1}^{t-1} \rho_{s, t-1}$.
	Without loss of generality, we assume $e_{l} <  t \le e_{l+1}$ for some $l \in [g]\cup\{0\}$.
	There, we have 
	\begin{equation}
	\label{eq:inter}
	\sum_{s=1}^{t-1} \rho_{s, t-1} = \left( t-e_{l} -1 \right)  + \sum_{i=1}^{l} s_i \rho^{l+1-i}.
	\end{equation}
	As a direct result, for any $t$, $\sum_{s=1}^{t-1} \rho_{s, t-1} \le \Delta + \Delta \frac{\rho}{1-\rho} = \frac{\Delta}{1-\rho}$.
	Similarly, by symmetry, we have for any $s < T$, $\sum_{t=s+1}^{T} \rho_{s, t-1}  \le \frac{\Delta}{1-\rho}$.
	Therefore, we have
	\begin{gather*}
	B_T \le \left( \frac{\Delta}{1-\rho} \right)^2 \\
	C_T \le  \left( \max_{s \in [T-1]} \sum_{t=s+1}^T  \rho_{s, t-1}\right)  \max_{t \in [T+1]}  \left( \sum_{l=1}^{t-1}  \rho_{l, t-1} \right) = \left( \frac{\Delta}{1-\rho} \right)^2.
	\end{gather*}

	Finally, for $A_T$, by using~\eqref{eq:inter}, we have
	\begin{align*}
	A_T &= \frac{1}{T}\sum_{t=1}^T\sum_{s=1}^{t-1} \rho_{s, t-1}\\
	& =\frac{1}{T}\sum_{l=0}^{g} \sum_{t=e_l+1}^{e_{l+1}}\sum_{s=1}^{t-1} \rho_{s, t-1}\\
	&= \frac{1}{T}\sum_{l=0}^{g} \sum_{t=e_l+1}^{e_{l+1}} \left[ \left( t-e_{l} -1 \right)  + \sum_{i=1}^{l} s_i \rho^{l+1-i} \right] \\
	&=\frac{1}{T}\sum_{l=0}^{g} \left[\frac{s_{l+1}(s_{l+1}-1) }{2} +  s_{l+1}\sum_{i=1}^{l} s_i \rho^{l+1-i} \right] \\
	&\le \frac{1}{T}\sum_{l=0}^{g} s_{l+1} \left( \frac{\Delta -1}{2}  + \Delta\sum_{i=1}^{l}\rho^{l+1-i}  \right) \\
	&\le \frac{\Delta -1}{2}  + \Delta\frac{\rho}{1-\rho} = \frac{ \Delta}{2} \frac{1+\rho}{1-\rho} - \frac{1}{2}.
	\end{align*}

\end{proof}

\section{Proof of LD-SGD with multiple D-SGDs}
\label{append:scheme1}

The task of analyzing convergence for different communication schemes $\IM_T$ can be reduced to figure out how residual errors are accumulated, i.e., to bound $A_T, B_T$ and $C_T$.
In this section, we are going to bound $A_T, B_T$ and $C_T$ when the update scheme is $\IM_T^1$.
To that end, we first give a technical lemma, which facilitate the computation.
Lemma~\ref{lem:rho} captures the accumulation rate of residual errors for $\IM_T^1$. 

\subsection{One technical lemma}
\begin{lem}[Manipulation on $\rho_{s, t-1}$]
	\label{lem:rho}
	When $\IM_T = \IM_T^1$,	the following properties hold for $\rho_{s, t-1}$:
	\begin{enumerate}
		\item
		$\rho_{s, t-1} = \prod_{l=s}^{t-1} \rho_l$ with $\rho_l = 1$ if $l$ mod $I \in [I_1]$, else $\rho_l = \rho$ where $I = I_1 + I_2$ and $\rho$ is defined in Assumption~\ref{asum:W}. As a direct consequence, $\rho_{s, t-1} = \rho_{s, l-1} \rho_{l, t-1}$ for any $s\le l \le t$.
		\item
		Define 
		\begin{equation}
		\label{eq:alpha}
		\alpha_j = \sum_{t=jI+1}^{(j+1)I} \sum_{s=1}^{t-1}\rho_{s, t-1}.
		\end{equation}
		Then for all $j \ge 0$,
		\begin{equation}
		\label{eq:alpha_u}
		\alpha_j \le \frac{1}{2}\left( \frac{1+\rho^{I_2}}{1-\rho^{I_2}} I_1^2 + \frac{1+\rho}{1-\rho} I_1 \right)  + I \frac{\rho}{1-\rho}.
		\end{equation}
		\item
		Define 
		\begin{equation}
		\label{eq:beta}
		\beta_j = \sum_{t=jI+1}^{(j+1)I} \sum_{s=1}^{t-1}\rho_{s, t-1}^2.
		\end{equation}
		Then for all $j \ge 0$,
		\begin{equation}
		\label{eq:beta_u}
		\beta_j \le \frac{1}{2}\left( \frac{1+\rho^{2I_2}}{1-\rho^{2I_2}} I_1^2 + \frac{1+\rho^2}{1-\rho^2} I_1 \right)  + I \frac{\rho^2}{1-\rho^2}.
		\end{equation}
		\item
		For any $t \ge 1$, $\sum_{s=1}^{t-1}\rho_{s, t-1} \le K$ where 
		\begin{equation}
		\label{eq:K}
		K = \frac{I_1}{1-\rho^{I_2}} + \frac{\rho}{1-\rho}.
		\end{equation}
		As a direct corollary, $\alpha_j \le IK$.
		\item 
		Define 
		\begin{equation}
		\label{eq:gamma}
		\gamma_j = \sum_{t=jI+1}^{(j+1)I} (\sum_{s=1}^{t-1}\rho_{s, t-1})^2.
		\end{equation}
		Then $\gamma_j \le K \alpha_j$, where $K$ is given in~\eqref{eq:K}.
		\item 
		Assume $T = (R+1)I$ for some non-negative integer $R$. Define
		\begin{equation}
		w_s = \sum_{t=s+1}^{T} \rho_{s, t-1}
		\end{equation}
		Then for all $s \in [T], w_s \le K$ where $K$ is given in~\eqref{eq:K}.
	\end{enumerate}
\end{lem}

\begin{proof}
	We prove these properties one by one:
	\begin{enumerate}
		\item It is a direct corollary of Lemma~\ref{lem:rho_compute}.
		\item We now directly compute $\alpha_j = \sum_{t=jI+1}^{(j+1)I} \sum_{s=1}^{t-1}\rho_{s, t-1}$. Without loss of generality, assume $t = jI + i$ with $j \ge 0, 1 \le i \le I$. (i) When $1 \le i \le I_1 +1$, then
		\begin{align}
		\label{eq:less}
		\sum_{s=1}^{t-1}& \rho_{s, t-1}  \nonumber  \\
		&= (i-1) + I_1 \sum_{r=0}^{j-1} \rho^{I_2(j-r)} + \sum_{r=0}^{j-1} \sum_{l=1}^{I_2} \rho^{I_2(j-r)+1-l}  \nonumber  \\
		&=(i-1) + I_1 \frac{\rho^{I_2} - \rho^{I_2(j+1)}}{1-\rho^{I_2}} + \frac{\rho - \rho^{jI_2+1}}{1-\rho}  \nonumber  \\
		& \le (i-1)  + I_1\frac{\rho^{I_2}}{1-\rho^{I_2}} + \frac{\rho}{1-\rho}.
		\end{align}
		(ii) When $I_1 + 1 \le i \le I$, then 
		\begin{align}
		\label{eq:large}
		\sum_{s=1}^{t-1}&\rho_{s, t-1} \nonumber \\
		& = \rho^{i - I_1 -1} \left[  I_1\sum_{r=0}^j \rho^{I_2(j-r)} + \sum_{r=0}^{j-1} \sum_{l=1}^{I_2} \rho^{I_2(j-r)+1-l} \right] \nonumber \\
		& \quad \quad  + \sum_{l=1}^{i-I_1-1} \rho^{i-I_1-l} \nonumber \\
		& =\rho^{i - I_1 -1}  \left[   I_1 \frac{1 - \rho^{I_2(j+1)}}{1-\rho^{I_2}} + \frac{\rho - \rho^{jI_2+1}}{1-\rho} \right] + \frac{\rho - \rho^{i-I_1}}{1-\rho} \nonumber \\
		&=\rho^{i - I_1 -1} \cdot  I_1 \frac{1 - \rho^{I_2(j+1)}}{1-\rho^{I_2}}  + \frac{\rho - \rho^{jI_2+i-I_1}}{1-\rho}  \nonumber \\
		&\le I_1  \frac{\rho^{i - I_1 -1} }{1-\rho^{I_2}} + \frac{\rho}{1-\rho}.
		\end{align}
		Therefore, by combining (i) and (ii), we obtain
		\begin{align*}
		\alpha_j &= \sum_{t=jI+1}^{(j+1)I} \sum_{s=1}^{t-1}\rho_{s, t-1} \\
		&\le \sum_{i=1}^{I_1} \left((i-1)  + I_1\frac{\rho^{I_2}}{1-\rho^{I_2}} + \frac{\rho}{1-\rho} \right) \nonumber \\
		& \quad \quad+ \sum_{i=I_1 + 1}^{I_1 + I_2} \left(I_1  \frac{\rho^{i - I_1 -1} }{1-\rho^{I_2}} + \frac{\rho}{1-\rho}\right)\\
		&= \frac{1}{2}\left( \frac{1+\rho^{I_2}}{1-\rho^{I_2}} I_1^2 + \frac{1+\rho}{1-\rho} I_1 \right)  + I \frac{\rho}{1-\rho}.
		\end{align*}
		\item Note that $\beta_j$'s share a similar structure with $\alpha_j$'s. Thus we can apply a similar argument in the proof of~\eqref{eq:alpha} to prove~\eqref{eq:beta}. A quick consideration reveals that~\eqref{eq:beta} can be obtained by replacing $\rho$ in~\eqref{eq:alpha} with $\rho^2$.
		\item Without loss of generality, assume $t = jI + i$ with $j \ge 0$ and $1 \le i \le I$. When $1 \le i \le I_1 +1$, from~\eqref{eq:less}, $\sum_{s=1}^{t-1}\rho_{s, t-1} \le (i-1)  + I_1\frac{\rho^{I_2}}{1-\rho^{I_2}} + \frac{\rho}{1-\rho} \le \frac{I_1}{1-\rho^{I_2}} + \frac{\rho}{1-\rho} = K$. When $I_1 + 1 \le i \le I_1 +I_2$, from~\eqref{eq:large}, $\sum_{s=1}^{t-1}\rho_{s, t-1} \le I_1  \frac{\rho^{i - I_1 -1} }{1-\rho^{I_2}} + \frac{\rho}{1-\rho} \le \frac{I_1}{1-\rho^{I_2}} + \frac{\rho}{1-\rho} = K$.
		\item The result directly follows from this inequality
		\begin{align*}
		\left( \sum_{s=1}^{t-1}\rho_{s, t-1}  \right)^2 
		&\le 
		\left(  \max_{t \ge 1}\sum_{s=1}^{t-1}\rho_{s, t-1}\right) \cdot
		\left( \sum_{s=1}^{t-1}\rho_{s, t-1}  \right)  \\
		&\le K \left( \sum_{s=1}^{t-1}\rho_{s, t-1}  \right)
		\end{align*}
		where $K$ is defined in~\eqref{eq:K}.
		\item Without loss of generality, assume $s = jI + i$ with $0 \le j \le R, 1 \le i \le I$. (i) We first consider the case where $1 \le i \le I_1 +1$, then
		\begin{align*}
		w_s & \le \sum_{t=s+1}^{T+1} \rho_{s, t-1}\\
		&= (I_1-i + 1) + \left( I_1 + \frac{\rho - \rho^{I_2+1}}{1-\rho} \right) \sum_{l=1}^{R-j} \rho^{I_2 l} + \sum_{l=1}^{I_2} \rho^{l}    \\
		&\le (I_1-i + 1)  + \left( I_1 + \frac{\rho - \rho^{I_2+1}}{1-\rho} \right) \frac{\rho^{I_2}}{1-\rho^{I_2}}\\
		& \qquad + \frac{\rho - \rho^{I_2+1}}{1-\rho}   \\
		& \le  \frac{I_1}{1-\rho^{I_2}} + \frac{\rho}{1-\rho} = K.
		\end{align*}
		(ii) Then consider the case where $I_1 + 1 \le i \le I$. If $R=j$, then $w_s = \sum_{l=1}^{I-i} \rho^{l} \le \frac{\rho}{1-\rho} \le K$. If $R \ge j + 1$, then
		\begin{align*}
		w_s & \le \sum_{t=s+1}^{T+1} \rho_{s, t-1}\\
		& = \rho^{I-i+1} \left( I_1 + \frac{\rho - \rho^{I_2+1}}{1-\rho} \right) \sum_{l=0}^{R-j-1} \rho^{I_2 l} + \sum_{l=1}^{I-i+1} \rho^{l}  \\
		& \le \frac{\rho^{I-i+1}}{1-\rho^{I_2}}  \left( I_1 + \frac{\rho - \rho^{I_2+1}}{1-\rho} \right)+ \frac{\rho - \rho^{I-i+2}}{1-\rho}  \\
		&= I_1\frac{\rho^{I-i+1}}{1-\rho^{I_2}} + \frac{\rho}{1-\rho} \\
		&\le \frac{I_1}{1-\rho^{I_2}} + \frac{\rho}{1-\rho} = K.
		\end{align*}
	\end{enumerate}
\end{proof}

\subsection{Proof of Theorem~\ref{thm:scheme1}}

\begin{proof}
	Without loss of generality, we assume $T$ is a multiplier of $I$, i.e., $T = (R+1)I$ for some positive integer $R$.
	This assumption will only simplify our proof but will not change the conclusion.
	
	For $A_T$, by using~\eqref{eq:beta_u} in Lemma~\ref{lem:rho}, we have
	\begin{align*}
	A_T &=\frac{1}{T}\sum_{t=1}^{T} \sum_{s=1}^{t-1} \rho^2_{s, t-1} \\
	&= \frac{1}{(R+1)} \sum_{j=0}^{R}  \frac{1}{I} \sum_{t=jI+1}^{(j+1)I}  \rho^2_{s, t-1} \\
	&=  \frac{1}{(R+1)} \sum_{j=0}^{R}  \frac{1}{I} \beta_{j} \\
	&\le \frac{1}{2I}\left( \frac{1+\rho^{2I_2}}{1-\rho^{2I_2}} I_1^2 + \frac{1+\rho^2}{1-\rho^2} I_1 \right)  + \frac{\rho^2}{1-\rho^2}.
	\end{align*}
	
	For $B_T$, by using the result about $\gamma_j$ in Lemma~\ref{lem:rho}, we have
	\begin{align*}
	B_T &= \frac{1}{T} \sum_{t=1}^{T} \left(\sum_{s=1}^{t-1} \rho_{s, t-1} \right)^2 \\
	&=\frac{1}{(R+1)} \sum_{j=0}^{R} \frac{1}{I} \sum_{t=jI+1}^{(j+1)I} \left(\sum_{s=1}^{t-1} \rho_{s, t-1} \right)^2   \\
	&=  \frac{1}{(R+1)} \sum_{j=0}^{R}  \frac{1}{I} \gamma_{j} \\
	&\overset{(a)}{\le} K \frac{1}{(R+1)} \sum_{j=0}^{R} \frac{1}{I}  \alpha_j\\
	&\overset{(b)}{\le} K \min\bigg\{ K,  \frac{1}{2I}\left( \frac{1+\rho^{I_2}}{1-\rho^{I_2}} I_1^2 + \frac{1+\rho}{1-\rho} I_1 \right)  + \frac{\rho}{1-\rho} \bigg\}
	\end{align*}
	where (a) follows the fact that $\gamma_{j}  \le K \alpha_j$ and (b) follows the two bounds on $\alpha_j$ provided in Lemma~\ref{lem:rho}.

	For $C_T$, by using the results about $w_s$ and $K$ in Lemma~\ref{lem:rho}, we have
	\begin{align*}
	C_T &= \max_{s \in [T-1]} \sum_{t=s+1}^T  \rho_{s, t-1}  \left( \sum_{l=1}^{t-1}  \rho_{l, t-1} \right)\\
	&\le \max_{t \in T-1} \sum_{t=s+1}^T  \rho_{s, t-1} \left( \max_{l} \sum_{l=1}^{t-1}  \rho_{l, t-1} \right)\\
	&\le K  \max_{t \in T-1} \sum_{t=s+1}^T  \rho_{s, t-1} \le K^2.
	\end{align*}
\end{proof}

\section{Proof of Theorem~\ref{thm:scheme2}}
\label{append:scheme2}
In this section, we will give the convergence result of Theorem~\ref{thm:scheme2} which states that the convergence will be fastened if we use the decaying strategy $\IM_T^2$.
Again, by using the framework introduced in Appendix~\ref{append:PD-SGD}, we only needs to give bounds for $A_T, B_T$ and $C_T$.
To that end, we need a modified version of Lemma~\ref{lem:rho} which reveals how the residual errors are accumulated for $\IM_T^2$.

Recall that to define $\IM_T^2$, we define an ancillary set
\[  \IM(I_1, I_2, M) = \{t \in [M(I_1+I_2)]: t \ \text{mod} \ (I_1+I_2) \notin [I_1] \}, \]
and then recursively define $\JM_0 = \IM(I_1, I_2, M)$ and 
\[ \JM_j =\IM\left( \bigg\lfloor \frac{I_1}{2^j} \bigg\rfloor , I_2, M\right) + \max(\JM_{j-1}), 1 \le j \le J \]
where $\max(\JM_{j-1})$ returns the maximum number collected in $\JM_{j-1}$ and $J = \lceil \log_2 I_1 \rceil $.
Finally we set
\[  \IM_T^2 = \cup_{j=0}^J \JM_{j} \cup [\max(\JM_{J}):T].  \]
In short, we first run $M$ rounds of LD-SGD with parameters $I_1$ and $I_2$, then run another $M$ rounds of LD-SGD with parameters $\lfloor \frac{I_1}{2} \rfloor$ and $I_2$, and keep this process going on until we reach the $J + 1$ th run, where $I_1$ shrinks to zero and we only run D-SGD.

\begin{lem}
	\label{lem:rho_strategy}
	Recall $M$ is the decay interval, $T$ the total steps and $\rho_{s, t-1} = \| \Ph_{s, t-1} - \Q\|$ where $\Ph_{s, t-1}$ is defined in~\eqref{eq:Ph} with $\W_t$ given in~\eqref{eq:W}. 
	Let $I_1, I_2$ be the initialized communication parameters. 
	Assume $T \ge \max \max(\JM_J)$. 
	Then for LD-SGD with the decaying strategy, we have that
	\begin{enumerate}
		\item $\frac{1}{T}\sum_{t=1}^T\sum_{s=1}^{t-1}\rho_{s, t-1} \le \frac{1}{T} \frac{I_1}{1-\rho^{I_2}} \rho^{T + I_2-\max(\JM_{J}) -1 } + (1-\frac{\max(\JM_{J})}{T}) \frac{\rho}{1-\rho}$;
		\item $\frac{1}{T}\sum_{t=1}^T\sum_{s=1}^{t-1}\rho_{s, t-1}^2 \le \frac{1}{T} \frac{I_1}{1-\rho^{2I_2}} \rho^{2(T + I_2-\max(\JM_{J}) -1)} + (1-\frac{\max(\JM_{J})}{T}) \frac{\rho^2}{1-\rho^2}$;
		\item 	For any $t \ge 1$, $\sum_{s=1}^{t-1}\rho_{s, t-1} \le K$ where $K = \frac{I_1}{1-\rho^{I_2}} + \frac{\rho}{1-\rho}$;
		\item 	For any $T> s \ge 1$, $\sum_{t=s+1}^{T} \rho_{s, t-1} \le K$.
	\end{enumerate}
	
\end{lem}
\begin{proof}
	
	We verify each inequality by directly computation:
	\begin{enumerate}
		\item 
		By exchanging the order of sum, we have 
		\item One can complete the proof by replacing $\rho$ with $\rho^2$ in the last argument.
		\item If $t \in \JM_j$, let $t_0 = \max(\JM_{j-1})$ (which is the largest element in $\JM_{j-1}$), then we have
		\begin{align*}
		&\sum_{s=1}^{t-1}\rho_{s, t-1} \\
		&\overset{(a)}{=} \sum_{l=0}^{j-1}  \sum_{s \in \JM_l}  \rho_{s, t_0} \rho_{t_0+1, t-1}  + \sum_{s=t_0+1}^{t-1} \rho_{s, t-1} \\
		&= \rho_{t_0+1, t-1} \sum_{l=0}^{j-1} \rho^{lI_2M} \sum_{s=0}^{M-1} \rho^{sI_2} \left(  \bigg\lfloor \frac{I_1}{2^l} \bigg\rfloor + \frac{\rho - \rho^{I_2+1}}{1-\rho}\right) \\
		& \quad +  \sum_{s=t_0+1}^{t-1} \rho_{s, t-1}\\
		&\overset{(b)}{\le} \rho_{t_0+1, t-1} \left(\frac{ 1- \rho^{kI_2M} }{1-\rho^{I_2}}I_1 + \frac{\rho(1- \rho^{kI_2M})}{1-\rho}\right) \\
		& \quad + \sum_{s=t_0+1}^{t-1} \rho_{s, t-1}\\
		&\overset{(c)}{\le} \frac{I_1}{1-\rho^{I_2}} + \frac{\rho}{1-\rho} = K
		\end{align*}
		where (a) uses the fact that $\rho_{s, t-1} = \rho_{s, t_0} \rho_{t_0+1, t-1}$; 
		(b) follows from $\lfloor \frac{I_1}{2^l} \rfloor  \le I_1$; 
		to obtain (c), one can conduct a similar discussion like what we have done in~Lemma~\ref{lem:rho} by discussing whether $t$ locates in the local update phase or the communication phrase. 
		The case is more complicated since we should also think about which round $t$ locates. 
		No matter which case here, (c) always holds. 
		\item The idea here is very similar to the that for the latest statement. 
		If $s \ge \max(\JM_J)$, then $\sum_{t=s+1}^{T} \rho_{s, t-1} \le \sum_{t=1}^{\infty} \rho^t = \frac{\rho}{1-\rho} \le K$. 
		Otherwise, local updates are involved in. 
		Similarly, one can imitate what we have done in Lemma~\ref{lem:rho} by discussing which round and which phase $s$ locates in. 
		Actually, this bound is rather rough.
	\end{enumerate}
\end{proof}

\begin{proof}[Proof of Theorem~\ref{thm:scheme2}.]
	By Lemma~\ref{lem:rho_strategy}, we have
	\begin{align*}
	A_T &= \frac{1}{T}\sum_{t=1}^T\sum_{s=1}^{t-1}\rho_{s, t-1}^2\\
	& \le \frac{1}{T} \frac{I_1}{1-\rho^{2I_2}} \rho^{2(T -\max(\JM_{J}))} + (1-\frac{\max(\JM_{J})}{T}) \frac{\rho^2}{1-\rho^2}\\
	B_T &=\frac{1}{T}\sum_{t=1}^T\left( \sum_{s=1}^{t-1}\rho_{s, t-1}\right)^2 \\
	&\le \left( \max_{t \in [T]}  \sum_{s=1}^{t-1}\rho_{s, t-1}  \right) \cdot  \frac{1}{T}\sum_{t=1}^T\left( \sum_{s=1}^{t-1}\rho_{s, t-1}\right)
	\\
	&\le  K \left[\frac{1}{T} \frac{I_1}{1-\rho^{I_2}} \rho^{T -\max(\JM_{J})} + (1-\frac{\max(\JM_{J})}{T}) \frac{\rho}{1-\rho}\right]\\
	C_T & \le  \left( \max_{t \in [T]}  \sum_{s=1}^{t-1}\rho_{s, t-1}  \right) \cdot
	\left(   \max_{s \in [T-1]}  \sum_{t=s+1}^{T}\rho_{s, t-1}    \right) \\
	&\le K^2.
	\end{align*}
	Then by combing Theorem~\ref{thm:PD-SGD}, we finish the proof of Theorem~\ref{thm:scheme2}.
	
\end{proof}

\section{Convergence of another update rule}
\label{append:another_update}

\subsection{Main result}
For completeness, in this section, we study another update rule in this section: 
\begin{equation}
\label{eq:new_update}
\X_{t+1} = \X_t\W_t - \eta \G(\X_t;\xi_t)
\end{equation}
where $\W_t$ is given in~\eqref{eq:W}. Since in this update rule, the stochastic gradient descent happens after each node communicates with its neighbors, we call this type of update as \textbf{communication-before}. By contrast, what we have analyzed in the body of this paper is termed as \textbf{communication-after}. A lot of previous efforts study the communication-before update rule, including~\cite{jiang2017collaborative,lian2017can}. 
Fortunately, our framework is so powerful that the convergence result for this new update rule can be easily parallel. 

\begin{theorem}[LD-SGD with any $\IM_T$ and the update rule~\eqref{eq:new_update}]
	\label{thm:PD-SGD-}
	Let Assumption~\ref{asum:smooth},~\ref{asum:within-var},~\ref{asum:inter-var},~\ref{asum:W} hold and the constants $L$, $\kappa$, $\sigma$, and $\rho$ be defined therein. 
	Let $\Delta = f(\overline{\x}_0) - \min_{\x } f (\x )$ be the initial error. For any fixed $T$, define
	\begin{gather*}
	\widehat{A}_T =\frac{1}{T}\sum_{t=1}^{T} \sum_{s=1}^{t-1} \rho^2_{s+1, t-1}, \\
	\widehat{B}_T = \frac{1}{T} \sum_{t=1}^{T} \left(\sum_{s=1}^{t-1} \rho_{s+1, t-1} \right)^2, \\
	\widehat{C}_T = \max_{s \in [T-1]} \sum_{t=s+1}^T  \rho_{s+1, t-1}  \left( \sum_{l=1}^{t-1}  \rho_{l+1, t-1} \right).
	\end{gather*} 
	If the learning rate $\eta$ is small enough such that
	\begin{equation}
	\eta \: \leq \: \min \bigg\{ \frac{1}{2L}, \; \frac{1}{4\sqrt{2}L\sqrt{\widehat{C}_{T}}} \bigg\} ,
	\end{equation}
	then
	\begin{equation}
	\frac{1}{T} \sum_{t=1}^{T} \EB \big\|\nabla f (\overline{\x}_t) \big\|^2 
	\: \leq \:
	{\underbrace{\frac{2\Delta}{\eta T} + \frac{\eta L\sigma^2}{n}}_{\text{fully sync SGD}}}
	+ 
	{\underbrace{\vphantom{ \left(\frac{a^{0.3}}{b}\right) } 4\eta^2L^2(\widehat{A}_T \sigma^2 + \widehat{B}_{T}\kappa^2}_{ \text{residual error}})} .
	\end{equation}
\end{theorem}

\begin{remark}
	Comparing the difference of results between Theorem~\ref{thm:PD-SGD} and Theorem~\ref{thm:PD-SGD-}, one can find that only the value of $A_T, B_T$ and $C_T$ have been modified. 
	In this way, one can parallel the conclusions derived for the update rule~\eqref{eq:pdsgd} to those with the update rule~\eqref{eq:new_update} by simply substituting $A_T, B_T$ and $C_T$ with $\widehat{A}_T, \widehat{B}_T$ and $\widehat{C}_T$. 
	
	Note that $\widehat{A}_T, \widehat{B}_T$ and $\widehat{C}_T$ is always no less than $A_T, B_T$ and $C_T$. 
	This indicators that the communication-after update rule~\eqref{eq:pdsgd} converges faster than the communication-before update rule~\eqref{eq:new_update}.
	As an example, we also given a counterpart of Theorem~\ref{thm:scheme1} in the latter section.
\end{remark}

\subsection{Useful lemmas and missing proof}

\begin{lem}[Residual error decomposition]
	\label{lem:residual-decom-2}
	Let $\X_1 = \x_1 \1_n^\top \in \RB^{d\times n}$ be the initialization, then for any $t \ge 2$,
	\begin{equation}
	\label{eq:decomposition-2}
	\X_t(\I-\Q)= - \eta \sum_{s=1}^{t-1} \G(\X_s;\xi_s) \left(  \Ph_{s+1, t-1} - \Q \right)
	\end{equation}
	where $\Ph_{s, t-1} $ is already given in~\eqref{eq:Ph}.
\end{lem}

\begin{proof}
	We still denote the gradient $\G(\X_t; \xi_t)$ as $\G_t$. According to the update rule, we have
	\begin{align*}
	&\X_{t}(\I_n-\Q) \\
	&=(\X_{t-1}\W_{t-1}-\eta \G_{t-1})(\I_n-\Q)\\
	&\overset{(a)}{=} \X_{t-1}(\I_n-\Q)\W_{t-1}-\eta\G_{t-1}(\I_n-Q)\\
	&\overset{(b)}{=} \X_{t-l}(\I_n-\Q)\prod_{s=t-l}^{t-1}\W_{s}-\eta\sum_{s=t-l}^{t-1}\G_{s}(\Ph_{s+1,t-1}-\Q)\\
	&\overset{(c)}{=}\X_{1}(\I_n-\Q)\Ph_{1, t-1}-\eta\sum_{s=1}^{t-1}\G_{s}(\Ph_{s+1,t-1}-\Q)
	\end{align*}
	where $(a)$ follows from $\W_{t-1}\Q=\Q\W_{t-1}$; (b) results by iteratively expanding the expression of $\X_{s}$ from $s=t-1$ to $s=t-l+1$ and plugging in the definition of $\Ph_{s,t-1}$ in \eqref{eq:Ph}; (c) follows simply by setting $l=t-1$. Finally, the conclusion follows from the assumption $\X_{1}(\I_n-\Q)=\0$.
\end{proof}

\begin{lem}[Bound on residual errors]
	\label{lem:residual-bound-2}
	Let $\rho_{s, t-1} = \|\Ph_{s, t-1} - \Q\|$ where $\Ph_{s, t-1}$ is defined in~\eqref{eq:Ph}. Then the residual error can be upper bounded, i.e., $V_t \le 2 \eta^2 U_t$ where
	\begin{equation*}
	U_t = \sigma^2  \sum_{s=1}^{t-1} \rho^2_{s+1, t-1} 
	+  \left( \sum_{s=1}^{t-1}  \rho_{s+1, t-1} \right) \left(  \sum_{s=1}^{t-1} \rho_{s+1, t-1} L_s\right).
	\end{equation*}
	where $L_s = 8L^2 V_s  + 4 \kappa^2 + 4 \EB \| \overline{\nabla f}(\X_s)  \|^2$.
\end{lem}

\begin{proof}
	The proof can be simply parallel by replacing $\rho_{s, t-1}$ with $\rho_{s+1, t-1}$ in Lemma~\ref{lem:residual-bound}.
\end{proof}

The next thing is to bound the average residual error, i.e., $\frac{1}{T}\sum_{t=1}^T V_t$.

\begin{lem}[Bound on average residual error]
	\label{lem:residual-average-2}
	For any fixed $T$, define
	\begin{gather}
		\label{eq:hatABC}
	\widehat{A}_T =\frac{1}{T}\sum_{t=1}^{T} \sum_{s=1}^{t-1} \rho^2_{s+1, t-1}, \\
	\widehat{B}_T = \frac{1}{T} \sum_{t=1}^{T} \left(\sum_{s=1}^{t-1} \rho_{s+1, t-1} \right)^2,\\
	\widehat{C}_T = \max_{s \in [T-1]} \sum_{t=s+1}^T  \rho_{s+1, t-1}  \left( \sum_{l=1}^{t-1}  \rho_{l+1, t-1} \right).
	\end{gather}
	Assuming the learning rate is so small that $16\eta^2L^2\widehat{C}_{T} < 1$, then
\begin{align*}
	\frac{1}{T}\sum_{t=1}^T V_t& \le  \left[ \widehat{A}_T \sigma^2 + \widehat{B}_T  \kappa^2 + \widehat{C}_T \frac{1}{T}\sum_{t=1}^{T}  \EB \| \overline{\nabla f}(\X_t)  \|^2  \right] \times  \\
& \quad \quad	\frac{8\eta^2}{1-16\eta^2L^2\widehat{C}_T}
\end{align*}
\end{lem}
\begin{proof}
	One can replace Lemma~\ref{lem:residual-bound} with Lemma~\ref{lem:residual-bound-2} in the proof of Lemma~\ref{lem:residual-average} to achieve the conclusion.
\end{proof}

\begin{proof}[Proof of Theorem~\ref{thm:PD-SGD-}]
	To prove Theorem~\ref{thm:PD-SGD-}, one can simply replace Lemma~\ref{lem:residual-average} with Lemma~\ref{lem:residual-average-2} in the proof of Appendix~\ref{append:main_proof}.
\end{proof}

\subsection{Example: Results of LD-SGD with $\IM_T^1$ and the other update rule}
In this part, we are going to bound $\widehat{A}_T, \widehat{B}_T$ and $\widehat{C}_T$ when the update scheme is $\IM_T^1$.
The result shown below shows the superiority of the original update~\eqref{eq:pdsgd}.

\begin{lem}[Manipulation on $\rho_{s+1,t-1}$]
	\label{lem:coro_rho_another}
	Noting that
	\begin{equation}
	\label{eq:rho_relation}
	\sum_{s=1}^{t-1}\rho_{s+1,t-1}=\sum_{s=2}^{t}\rho_{s,t-1}\le\sum_{s=1}^{t}\rho_{s,t-1}=\sum_{s=1}^{t-1}\rho_{s,t-1}+1,
	\end{equation}
	When $\IM_T=\IM_T^1$, we can immediately deduce the following properties from Lemma~\ref{lem:rho} for $\rho_{s+1,t-1}$.
	\begin{enumerate}
		\item$\tilde{\alpha}_{j}=\sum_{jI+1}^{(j+1)I}\sum_{s=1}^{t-1}\rho_{s+1, t-1}\le\frac{1}{2}\left( \frac{1+\rho^{I_2}}{1-\rho^{I_2}} I_1^2 + \frac{1+\rho}{1-\rho} I_1 \right)  + I \frac{1}{1-\rho}$. \\
		\item $\tilde{\beta}_j = \sum_{t=jI+1}^{(j+1)I} \sum_{s=1}^{t-1}\rho_{s+1, t-1}^2\le\frac{1}{2}\left( \frac{1+\rho^{2I_2}}{1-\rho^{2I_2}} I_1^2 + \frac{1+\rho^2}{1-\rho^2} I_1 \right)  + I \frac{1}{1-\rho^2}$.\\
		\item Let $\tilde{K} = \frac{I_1}{1-\rho^{I_2}} + \frac{1}{1-\rho}=K+1$, then $\sum_{s=1}^{t-1}\rho_{s+1, t-1} \le \tilde{K}$ and $\tilde{\alpha}_{j} \le I \tilde{K}$.  \\
		\item  $\tilde{\gamma}_j = \sum_{t=jI+1}^{(j+1)I} (\sum_{s=1}^{t-1}\rho_{s+1, t-1})^2\le \tilde{K}\tilde{\alpha}_{j}$.\\
		\item If $T = (R+1)I$, we have $\tilde{w}_{s}=\sum_{t=s+1}^{T}\rho_{s+1,t-1}=1+\sum_{t=s+2}^{T}\rho_{s+1, t-1}=1+w_{s+1}\le 1+K=\tilde{K}$.
	\end{enumerate}
\end{lem}

\begin{theorem}[LD-SGD with $\IM_T^1$ and the update rule~\eqref{eq:new_update}]
	\label{thm:scheme1-}
	When we set $\IM_T = \IM_T^1$ for LD-SGD with update rule~\eqref{eq:new_update}, under the same setting, Theorem~\ref{thm:PD-SGD-} holds with
	\begin{gather*}
	\widehat{A}_T \le \frac{1}{2I}\left( \frac{1+\rho^{2I_2}}{1-\rho^{2I_2}} I_1^2 + \frac{1+\rho^2}{1-\rho^2} I_1 \right)  +  \frac{1}{1-\rho^2}, \\
	\max \left\{  \widehat{B}_T, \widehat{C}_T \right\}\le \tilde{K}^2, \tilde{K}= \frac{I_1}{1 - \rho^{I_2}} + \frac{1}{1-\rho}.
	\end{gather*}
	Therefore, LD-SGD converges with $\IM_T^1$ and the update rule~\eqref{eq:new_update}.
\end{theorem}

\section{Discussion on others' convergence results}
\label{appen:discussion}

\begin{table*}[ht]
	\setlength{\tabcolsep}{0.3pt}
	\centering
	\caption{Convergence results of LD-SGD with $\IM_T^1$.
		To recover previous algorithms, $I_1, I_2$ and $\rho$ (which is defined by Assumption~\ref{asum:W}) are determined as following. 
		The result is directly obtained by combining the bounds derived in Theorem~\ref{thm:scheme1} with Theorem~\ref{thm:PD-SGD} or Theorem~\ref{thm:PD-SGD-}. 
		In this table, $\Delta = f(\overline{\x}_0) - \min_{\x } f (\x )$ is the initial error, $\eta$ the learning rate and $I = I_1 + I_2$. 
		The result for D-SGD is obtained from Theorem~\ref{thm:PD-SGD-} while the rest from Theorem~\ref{thm:PD-SGD}.}
	\label{tab:algo_result}
	\begin{tabular}{ccccccc}
		\toprule
		Algorithms & ~~~~~$I_1$~~~~~ & ~~~~~$I_2$~~~~~ 
		& ~~~~~$\rho$~~~~~ &~~~~~$\sigma$~~~~~ &~~~~~$\kappa$~~~~~ &
		Convergence Rate \\
		\midrule
		SGD~\cite{bottou2018optimization}& 0 & 1 & $0$ & 0 & 0 &
		$\frac{2\Delta}{\eta T} + \frac{\eta L\sigma^2}{n}$ \\
		PR-SGD~\cite{yu2019parallel} & $ \geq 1$ & $1$ &  $0$ & $>0$&  $>0$ & 
		$\frac{2\Delta}{\eta T} + \frac{\eta L\sigma^2}{n}  + 8\eta^2L^2\sigma^2 I_1 + 16\eta^2L^2\kappa^2 I_1^2$ \\
		D-SGD~\cite{lian2017can} & 0 & 1 & $[0, 1)$ & $>0$ & $>0$ &
		$\frac{2\Delta}{\eta T} + \frac{\eta L\sigma^2}{n} +  \frac{16\eta^2L^2\sigma^2}{1-\rho^2} + \frac{ 16\eta^2L^2\kappa^2}{(1-\rho)^2}$  \\
		PD-SGD~\cite{wang2018cooperative}&$\geq 0$ & $1$ & $[0,1)$ & $>0$& 0 &
		$\frac{2\Delta}{\eta T} + \frac{\eta L\sigma^2}{n} + 8\eta^2L^2 \sigma^2(\frac{1+\rho^2}{1-\rho^2} I -1)$ \\
		\bottomrule
	\end{tabular}
\end{table*}

As discussed in Section~\ref{sec:algo}, LD-SGD with the update scheme $\IM_T^1$ incorporates many previous algorithms by setting $I_1, I_2$ and $\rho$ correspondingly.
Taking the difference between the update rule~\eqref{eq:pdsgd} and~\eqref{eq:new_update} into account, we could give convergence results for previous algorithms via Theorem~\ref{thm:PD-SGD} or Theorem~\ref{thm:PD-SGD-} (see Table~\ref{tab:algo_result}).
In this section, we compare the result obtained from our analytical framework with their original ones.

\subsection{Convergence for PR-SGD} 
PR-SGD~\cite{zhou2017convergence,yu2019parallel,yu2019linear} is the special case of LD-SGD when $I_2 =1$ and $\rho = 0$ (i.e., $\W = \Q = \frac{1}{n}\1_n\1_n^\top$).~\cite{yu2019parallel} derives its convergence (Theorem~\ref{thm:yu}) by requiring Assumption~\ref{asum:bound} which is definitely stronger than our Assumption~\ref{asum:inter-var}. Roughly speaking we always have bound $\kappa^2 \le 4G^2$ since $	\frac{1}{n}\sum_{k=1}^n \big\| \nabla f_k(\x) - \nabla f(\x) \big\|^2 \le \frac{2}{n} \sum_{k=1}^n \|  \nabla f_k(\x) \|^2 + 2 \|\nabla f(\x)\|^2 \le 4G^2$. Then our bound matches theirs up to constant factors. Another interesting thing is in this case our bound only depends on $I_1 = I-1$ while~\cite{yu2019parallel}'s relies on $I$. Though they are the same asymptotically, our refined analysis shows that the step of model averaging doesn't account for the accumulation of residual errors.

\begin{assumption}(Bounded second moments)
	\label{asum:bound}
	There are exist some $G > 0$ such that for all $k \in [n]$,
	\[ \EB_{\xi \sim \DM_k} \|  \nabla F_k(\x; \xi)\|^2 \le G^2. \]
\end{assumption}

\begin{theorem}[\cite{yu2019parallel}]
	\label{thm:yu}
	Let Assumption~\ref{asum:smooth},~\ref{asum:within-var} and~\ref{asum:bound} hold and $L, \sigma, G$ defined therein. 
	Let $\{\x_t\}_{t=1}^T$ denote by the sequence obtained by PR-SGD and $\Delta = f(\overline{\x}_0) - \min_{\x } f (\x )$ be the initial error.
	If $0 < \eta  \le \frac{1}{L}$, then for all $T$, we have
	\begin{equation*}
	\frac{1}{T} \sum_{t=1}^{T} \EB \big\|\nabla f (\overline{\x}_t) \big\|^2 
	\: \leq \:
	\frac{2\Delta}{\eta T} + \frac{\eta L\sigma^2}{n} +  4\eta^2 I^2 G^2L^2.
	\end{equation*}
\end{theorem}

\subsection{Convergence for D-SGD} 
D-SGD~\cite{jiang2017collaborative,lian2017can} is the special case of LD-SGD where $I_1 =0, I_2=1$, $1  > \rho \ge 0$ and the communication-after update rule (introduced in Appendix~\ref{append:another_update}) is applied. 
The original paper~\cite{lian2017can} provides an analysis for D-SGD, which we simplify and translate into Theorem~\ref{thm:lian} in our notation. 
To guarantee convergence at a neighborhood of stationary points,~\cite{lian2017can} requires a smaller learning rate $\OM(\frac{1-\rho}{\sqrt{n}L})$ than our $\OM(\frac{1-\rho}{L})$. 
By contrast their residual error is larger than ours up to a factor of $\OM(n)$. They could achieve as similar bounds on residual errors as ours by shrinking the learning rate, but the convergence would be slowed down.

\begin{theorem}[\cite{lian2017can}]
	\label{thm:lian}
	Let Assumption~\ref{asum:smooth},~\ref{asum:within-var}, ~\ref{asum:inter-var} and~\ref{asum:W} hold and $L, \sigma, \kappa$ defined therein. 
	Let $\{\x_t\}_{t=1}^T$ denote by the sequence obtained by D-SGD and $\Delta = f(\overline{\x}_0) - \min_{\x } f (\x )$ be the initial error.
	When the learning rate is small enough\footnote{In this way, their $D_2 \ge \frac{2}{3}$ and $D_1 \ge \frac{1}{4}$, and this result follows from replacing $D_1, D_2$ with these constant lower bounds.} such that $\eta \le \frac{1-\rho}{3\sqrt{6}L} \frac{1}{\sqrt{n}}$, then for all $T$, we have
	\begin{align*}
	\frac{1}{T} \sum_{t=1}^{T} &\EB \big\|\nabla f (\overline{\x}_t) \big\|^2  \\
	&\: \leq \: 
	\frac{4\Delta}{\eta T} + \frac{2\eta L\sigma^2}{n} + n \left[ \frac{6\eta^2L^2\sigma^2 }{1-\rho^2} +  \frac{54\eta^2L^2\kappa^2 }{(1-\rho)^2} \right].
	\end{align*}
\end{theorem}

\subsection{Convergence for PD-SGD}
The PD-SGD is derived as a byproduct of the framework of Cooperative SGD (C-SGD) in~\cite{wang2018cooperative}. 
In that paper,~\cite{wang2018cooperative} terms PD-SGD as Decentralized Periodic Averaging SGD (PDA-SGD).
In our paper, the PD-SGD (or DPA-SGD) is the case of LD-SGD with $\IM_T^1$ when $I_2 = 1$ and $1 > \rho \ge 0$.
We translate their original analysis into Theorem~\ref{thm:wang} for ease of comparison.

First, our residual error is exactly the same with theirs up to constant factors. Second, they didn't consider the case when the data is non-identically distributed. Third, we allow more flexible communication pattern design by introducing parameters $I_2$.

\begin{theorem}[\cite{wang2018cooperative}]
	\label{thm:wang}
	Let Assumption~\ref{asum:smooth},~\ref{asum:within-var} and~\ref{asum:W} hold and $L, \sigma$ defined therein. 
	Let $\{\x_t\}_{t=1}^T$ denote by the sequence obtained by PD-SGD and $\Delta = f(\overline{\x}_0) - \min_{\x } f (\x )$ be the initial error.
	When the learning rate is small enough such that $\eta \le \min  \{  \frac{1}{2L}, \frac{1-\rho}{\sqrt{10}LI} \}$, then for all $T$, we have
	\begin{equation*}
	\frac{1}{T} \sum_{t=1}^{T} \EB \big\|\nabla f (\overline{\x}_t) \big\|^2 
	\: \leq \:
	\frac{2\Delta}{\eta T} + \frac{\eta L\sigma^2}{n} + \eta^2L^2 \sigma^2(\frac{1+\rho^2}{1-\rho^2} I -1).
	\end{equation*}
\end{theorem}

\section{Experiments Details}
\label{appen:exp_detail}
Our experiment setting follows~\cite{wang2019matcha} closely and is implemented with PyTorch and MPI4Py.
We adopt the code released by~\cite{wang2019sharper}.

\paragraph{Image Classification.}
CIFAR-10 and CIFAR-100 consist of 60, 000 color images in 10 and 100 classes, respectively.
We set the initial learning rate as 0.8 and it decays by 10 after 100 and 150 epochs.
The mini-batch size per worker node is 64.

\paragraph{Language Modeling.}
The PTB dataset contains 923, 000 training words. 
A two-layer LSTM with 1500 hidden nodes in each layer~\cite{press2016using} is adopted.
We set the initial learning rate as 20 and it decays by 4 when the training procedure saturates. 
The mini-batch size per worker node is 10. 
The embedding size is 1500. 
All algorithms are trained for 40 epochs.

\paragraph{Machines.}
The training procedure is performed in a single machine which is equipped with four GPUs (Tesla P100 PCIe 16GB).
We uniformly distribute models (as well as the associated data) into the four GPUs.
Therefore, the communication in our paper actually means the inter-GPU communication.
If we use a realistic wireless environment, then the communication cost will be larger than that of our current experiments.
Hence, the advantage of LD-SGD in real time will be much larger.

\begin{figure*}[ht]
	\centering
	\vspace{-0.1in} 
	\subfloat[Graph 0 $(\rho= 0.808)$]{
		\includegraphics[width=0.33\textwidth] {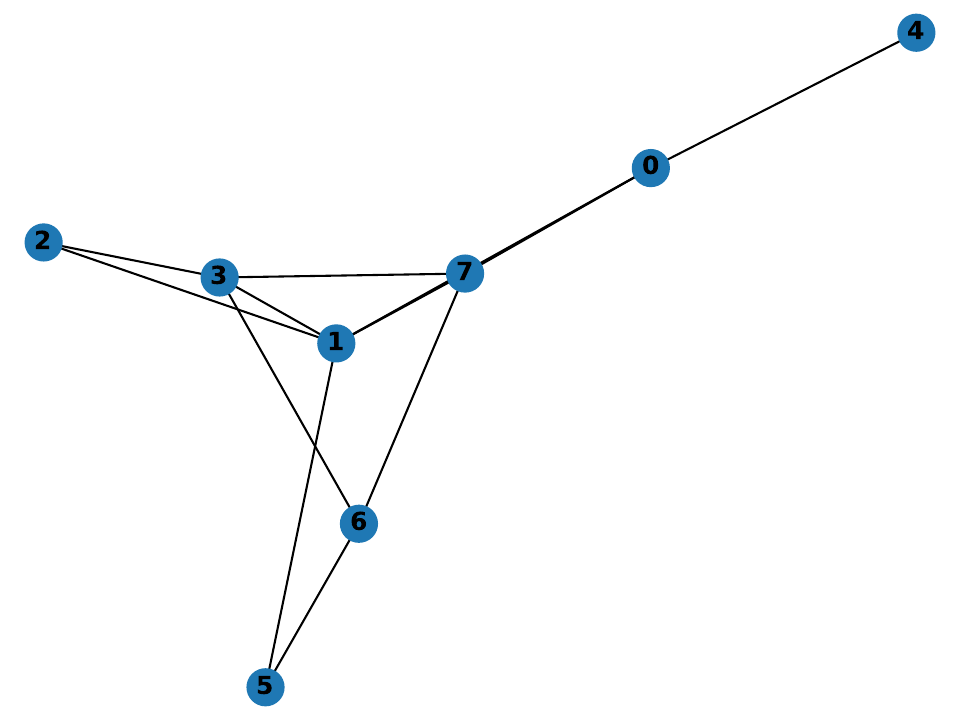}}
	\hspace{-0.1in} 
	\subfloat[Graph 1 $(\rho= 0.964)$]{
		\includegraphics[width=0.33\textwidth] {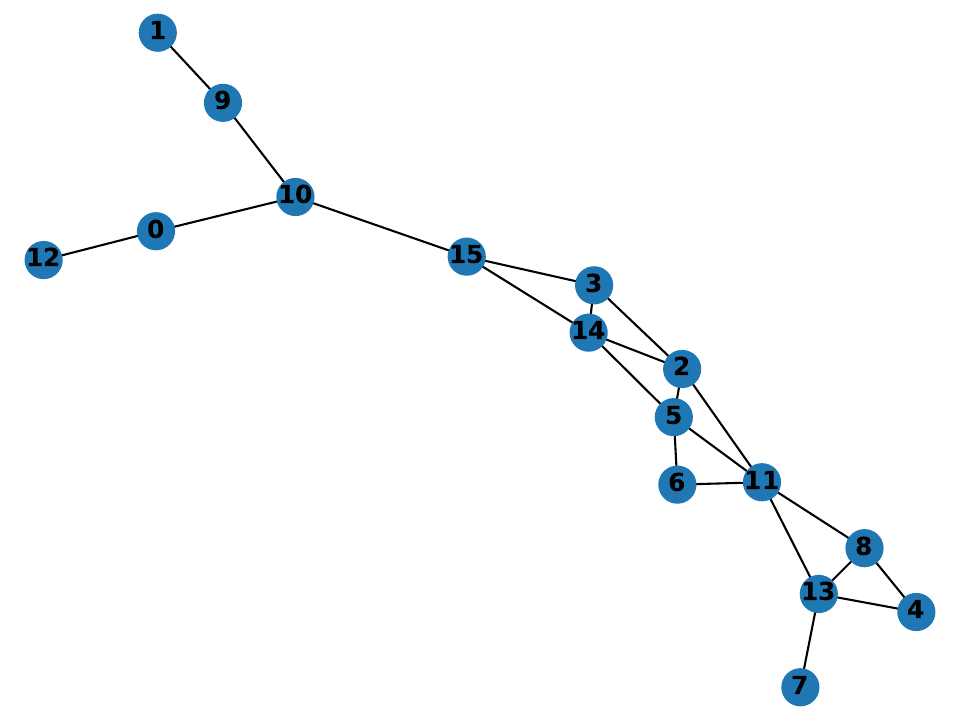}}
	\hspace{-0.1in} 
	\subfloat[Graph 2 $(\rho= 0.693)$]{
		\includegraphics[width=0.33\textwidth] {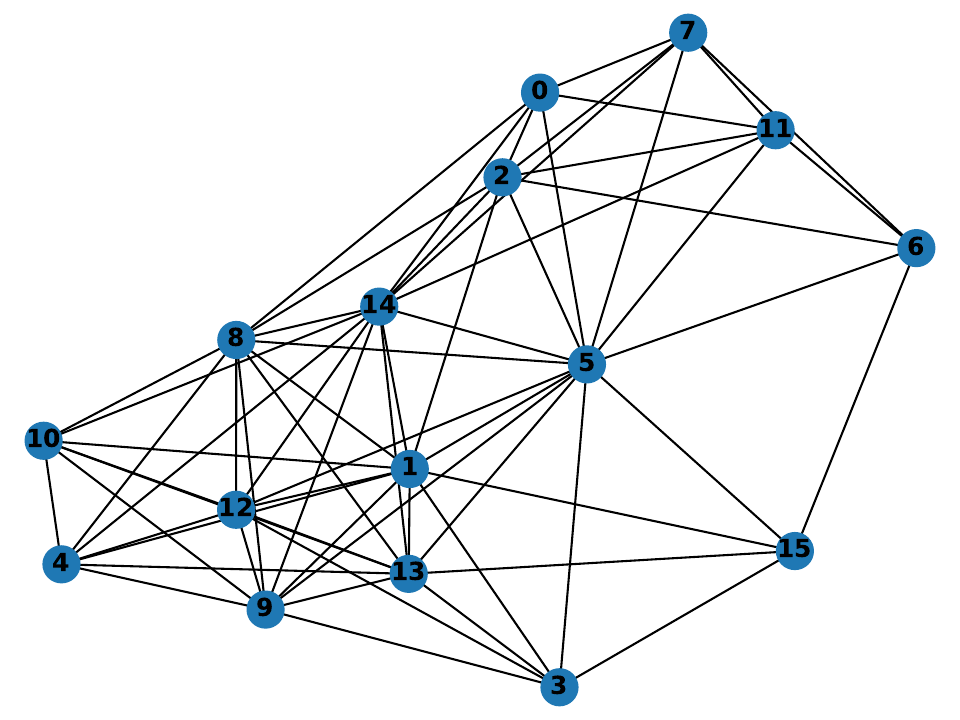}}
	\caption{Three topologies that will be used in following experiments.}
	\label{fig:graph}
	\vspace{-0.1in} 
\end{figure*}


\begin{figure*}[ht!]
	\centering
	\subfloat[different $I_1/I_2$ on CIFAR 10]{
		\includegraphics[width=0.3\textwidth] {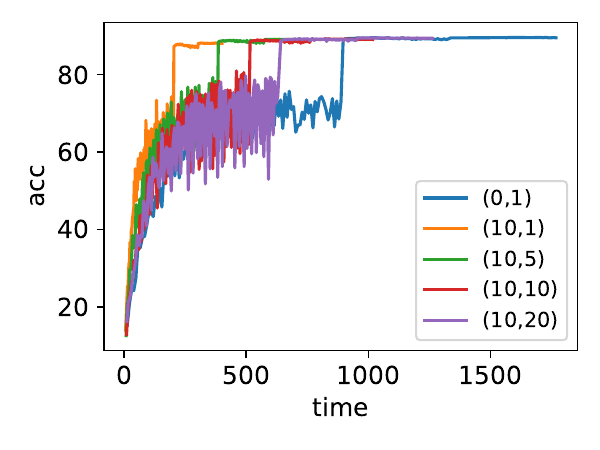}}
	\hspace{-0.15in} 
	\subfloat[same $I_1/I_2$ on CIFAR10]{
		\includegraphics[width=0.3\textwidth] {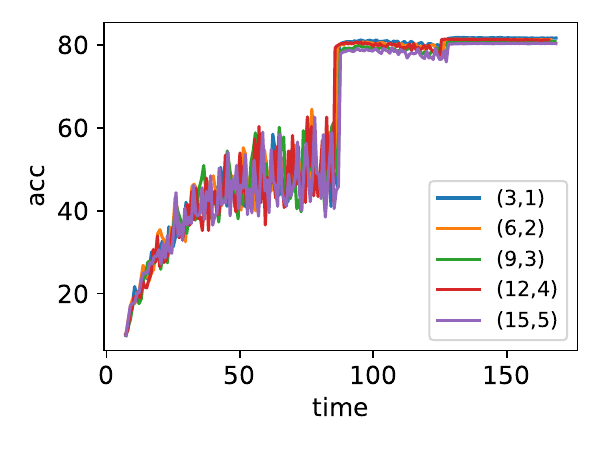}}
	\hspace{-0.15in} 
	\subfloat[decay strategy on CIFAR10]{
		\includegraphics[width=0.3\textwidth] {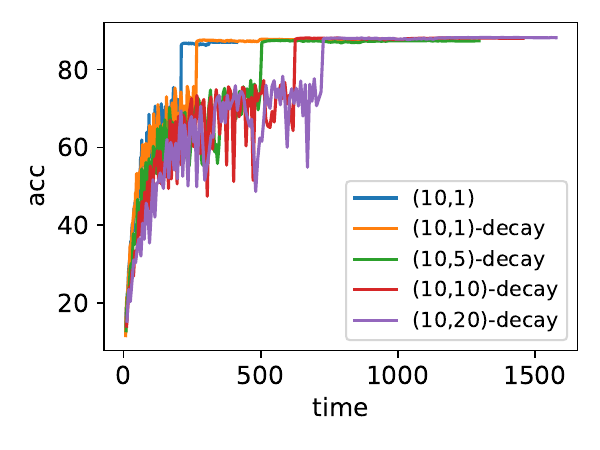}}\\
	\subfloat[different $I_1/I_2$ on CIFAR10]{
		\includegraphics[width=0.3\textwidth] {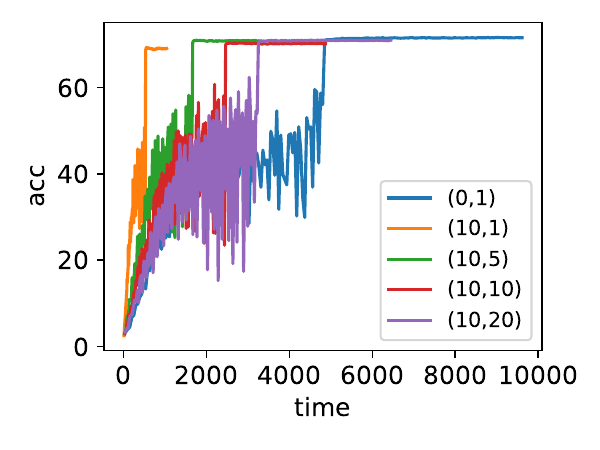}}
	\subfloat[same $I_1/I_2$ on CIFAR100]{
		\includegraphics[width=0.3\textwidth] {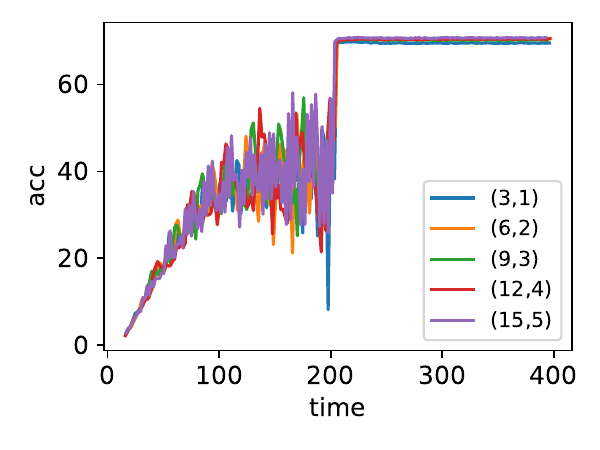}}
	\subfloat[decay strategy  on CIFAR100]{
		\includegraphics[width=0.3\textwidth] {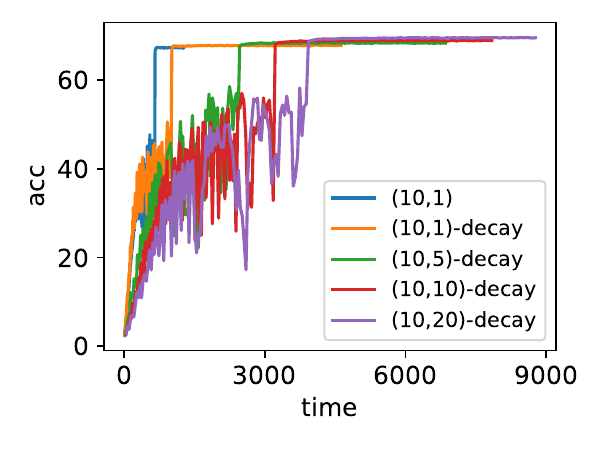}}\\
	\caption{Test accuracy of LD-SGD with different $(I_1, I_2)$. 
	The first column shows the results of different $I_1/I_2$, the column row shows show those of the same $I_1/I_2$, and the rightmost column  shows those of the decay strategy.
The first row is the results on CIFAR 10, while the second row in on CIFAR 100.}
	\label{fig:iid-test-acc}
	\vspace{-0.1in} 
\end{figure*}

\section*{Acknowledgment}
Li, Yang and Zhang have  been  supported  by  the National Key Research and Development Project of MOST of China (No. 2018AAA0101004),  Beijing  Natural Science Foundation (Z190001)  and Beijing Academy of Artificial Intelligence (BAAI).


%
%

%

%
%
%




\end{document}